\documentclass[mlmain,onecolumn]{jmlr}

\usepackage{tikz}
\usetikzlibrary{arrows.meta}
\usepackage{enumerate}
\usepackage{amssymb}
\usepackage{amsmath}
\usepackage{graphicx}
\usepackage{extarrows}
\usepackage{endnotes}
\usepackage{verbatim}
\usepackage{color}
\usepackage{comment}
\usepackage{tikz}
\usepackage{quiver}
\usepackage{url}
\usepackage{float}
\providecommand{\qedhere}{\tag*{\jmlrBlackBox}}
\makeatletter
\renewcommand*{\@jmlrenddoc}{\FloatBarrier}
\makeatother

\firstpageno{1}
\jmlryear{2026}
\jmlrworkshop{Symmetry and Geometry in Neural Representations}

\title[Persistence, Exposure, and Credit Routing in Linear tPC]{Persistence, Exposure, and Credit Routing in Linear Temporal Predictive Coding}

\author{\Name{Akira Tanase} \Email{akira.tanase23@ic.ac.uk}\\
\addr Sakana AI\\
\addr Imperial College London
\AND
\Name{Julian Gould} \Email{juliangould@sakana.ai}\\
\addr Sakana AI
\AND
\Name{Jeffrey Seely} \Email{jeffrey@sakana.ai}\\
\addr Sakana AI}

\begin{document}

\maketitle

\begin{abstract}
 Temporal Predictive Coding (tPC) networks offer a layer-local alternative to backpropagation. For linear tPC, we show, by applying tools from cellular sheaf theory, that temporal credit assignment is a tug-of-war between spatial displacement and temporal errors, whose ratios are modulated by persistence and exposure operators $P$ and $Q$. Spatial inference therefore gives us the effective recurrent transition $F=PA$, whilst future credit is filtered by $P^\ast$. This creates a persistence--error tradeoff in spatial modes: under non-amplifying invariant recurrence, modes that preserve influence over long delays simultaneously fail to expose that influence strongly as temporal error. Recurrence can circumvent this via routing between persistent and exposed modes.
\end{abstract}


\section{Introduction}
Predictive coding networks are a layer-local alternative to backpropagation for training hierarchical neural networks \citep{rao1999predictive,whittington2017approximation,salvatori2026survey}. In temporal predictive coding, we learn sequential data by adding recurrent connections between corresponding layers \citep{millidge2024temporal,ngkeekwong2026complex}. We focus on linear tPC for autoregressive sequence modelling of observations $o_0,\dots,o_T$, processed sequentially through time. At time $t$, the current observation is fixed at the bottom layer, $z_{0,t}=o_t$, whilst the remaining activities are inferred using both the current observation and temporal predictions from the previously inferred slice; the resulting state is then used to predict time $t+1$. Formally, we have activities $z_{\ell,t}\in H_\ell$ indexed across layers $\ell=0,\dots,L$ and time $t=0,\dots,T$. Within each time slice, $W_\ell:H_{\ell+1}\to H_\ell$ maps higher-layer activity $z_{\ell+1,t}$ to a prediction of $z_{\ell,t}$, whilst across time $A_\ell:H_\ell\to H_\ell$ maps $z_{\ell,t}$ to a prediction of $z_{\ell,t+1}$. The maps $W_\ell$ and $A_\ell$ are shared across time. These maps define receiver-side spatial and temporal edge discrepancies $r^s_{\ell,t}=z_{\ell,t}-W_\ell z_{\ell+1,t}$ and $r^\tau_{\ell,t}=z_{\ell,t+1}-A_\ell z_{\ell,t}$.

\begin{figure}[t]
  \centering
  \begin{tikzpicture}[
  x=3.25cm,
  y=1.15cm,
  >=Latex,
  activation/.style={inner sep=3pt},
  map/.style={->, thin},
  temporal label/.style={font=\small, inner sep=1.5pt},
  spatial label/.style={font=\small, inner sep=1pt,
    anchor=east, xshift=-5pt}
]
  \fill[black!4] (0.54,0.40) rectangle (1.46,4.42);

  \node[activation] (zm1L) at (0,4) {$H_L$};
  \node[activation] (ztL)  at (1,4) {$H_L$};
  \node[activation] (zp1L) at (2,4) {$H_L$};

  \node[activation] (zm1Lm1) at (0,3) {$H_{L-1}$};
  \node[activation] (ztLm1)  at (1,3) {$H_{L-1}$};
  \node[activation] (zp1Lm1) at (2,3) {$H_{L-1}$};

  \node at (0,2.4) {$\vdots$};
  \node at (1,2.4) {$\vdots$};
  \node at (2,2.4) {$\vdots$};

  \node[activation] (zm11) at (0,1.8) {$H_1$};
  \node[activation] (zt1)  at (1,1.8) {$H_1$};
  \node[activation] (zp11) at (2,1.8) {$H_1$};

  \node[activation] (zm10) at (0,0.8) {$H_0$};
  \node[activation] (zt0)  at (1,0.8) {$H_0$};
  \node[activation] (zp10) at (2,0.8) {$H_0$};

  \draw[map] (zm1L) -- node[midway, spatial label] {$W_{L-1}$} (zm1Lm1);
  \draw[map] (ztL)  -- node[midway, spatial label] {$W_{L-1}$} (ztLm1);
  \draw[map] (zp1L) -- node[midway, spatial label] {$W_{L-1}$} (zp1Lm1);
  \draw[map] (zm11) -- node[midway, spatial label] {$W_0$} (zm10);
  \draw[map] (zt1)  -- node[midway, spatial label] {$W_0$} (zt0);
  \draw[map] (zp11) -- node[midway, spatial label] {$W_0$} (zp10);

  \draw[map] (zm1L) -- node[temporal label, above] {$A_L$} (ztL);
  \draw[map] (ztL)  -- node[temporal label, above] {$A_L$} (zp1L);
  \draw[map] (zm1Lm1) -- node[temporal label, above] {$A_{L-1}$} (ztLm1);
  \draw[map] (ztLm1)  -- node[temporal label, above] {$A_{L-1}$} (zp1Lm1);
  \draw[map] (zm11) -- node[temporal label, above] {$A_1$} (zt1);
  \draw[map] (zt1)  -- node[temporal label, above] {$A_1$} (zp11);
  \draw[map] (zm10) -- node[temporal label, above] {$A_0$} (zt0);
  \draw[map] (zt0)  -- node[temporal label, above] {$A_0$} (zp10);

  \node[font=\small] at (0,4.65) {time $t-1$};
  \node[font=\small\bfseries] at (1,4.65) {time $t$};
  \node[font=\small] at (2,4.65) {time $t+1$};
  \node[font=\small, anchor=east] at (-0.62,4) {$\ell=L$};
  \node[font=\small, anchor=east] at (-0.62,3) {$\ell=L-1$};
  \node[font=\small, anchor=east] at (-0.62,1.8) {$\ell=1$};
  \node[font=\small, anchor=east] at (-0.62,0.8) {$\ell=0$};
\end{tikzpicture}
  \caption{Linear temporal predictive coding architecture. Spatial maps
  $W_\ell$ predict down each time slice, whilst temporal
  maps $A_\ell$ predict forward across time. At time $t$,
  the observation $z_{0,t}=o_t$ is fixed and inference updates the hidden activities in the shaded slice.}
  \label{fig:tpc-setup}
\end{figure}
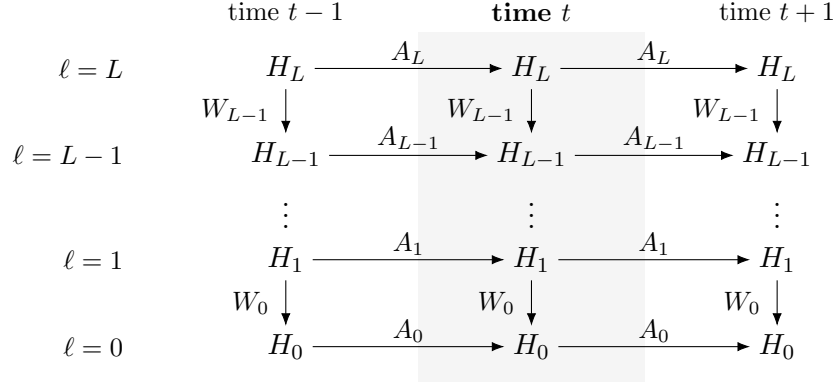

Training proceeds sequentially through time. At each time $t$, during \emph{inference}, the weights and the previously inferred state $z_{\bullet,t-1}^\star$ are kept fixed, the current observation $z_{0,t}=o_t$ is clamped, and only the remaining activities $z_{1,t},\dots,z_{L,t}$ are updated to minimise the current slice energy. This energy aggregates within-slice spatial residuals with the incoming temporal residuals from the transition $t-1\to t$. The resulting inferred state is denoted by $z_{\bullet,t}^\star$ and is used to generate temporal predictions for time $t+1$ \citep{millidge2024temporal}. After the full trajectory has been inferred, during \emph{learning}, the inferred activities are held fixed and the same shared maps are updated using the residuals and incident activities accumulated across time: each $W_\ell$ from the spatial residuals in each slice, and each $A_\ell$ from the temporal residuals on the incoming transitions.

Training is a tug-of-war between spatial and temporal predictions. By constructing a directed network sheaf over the tPC computational graph (\hyperref[sec:directed-network-sheaves]{Section~\ref*{sec:directed-network-sheaves}}), we realise tPC as directed sheaf diffusion (\hyperref[sec:directed-sheaf-diffusion]{Section~\ref*{sec:directed-sheaf-diffusion}}). We then show in \hyperref[sec:spatial-inference]{Section~\ref*{sec:spatial-inference}} that exact inference finds a compromise between the spatial prediction and temporal proposal, modulated by operators \(P\) and \(Q\), giving the effective recurrent map \(F=PA\). In \hyperref[sec:spatial-modes]{Section~\ref*{sec:spatial-modes}}, we analyse this \(P,Q\) split along individual eigenmodes of the spatial sheaf Laplacian, finding that low-penalty modes are preferentially retained whilst high-penalty modes are preferentially exposed as temporal error. In \hyperref[sec:temporal-credit]{Section~\ref*{sec:temporal-credit}}, we show that the future temporal credit omitted by standard tPC learning is filtered by \(P^\ast\), whilst \hyperref[sec:long-term-routing]{Section~\ref*{sec:long-term-routing}} shows how recurrence can route influence between persistent and exposed modes over long time horizons. Finally, in \hyperref[sec:delayed-recall-example]{Section~\ref*{sec:delayed-recall-example}}, we validate mode routing on a single-hidden-layer network.

\section{Directed Network Sheaves}
\label{sec:directed-network-sheaves}

Cellular sheaves \citep{curry2014sheaves} are a powerful framework for local-to-global organisation and consensus seeking on graphs \citep{hansen2019spectral,bodnar2022neural,hansen2021opinion}. We use the directed cellular sheaf construction of \citet{ribeiro2026cooperative}.

\begin{definition}[Directed network sheaf]
Given a directed graph $G=(V,E)$, a directed network sheaf $(\mathcal F,S,T)$ assigns vector spaces $\mathcal F(v)$ and $\mathcal F(e)$, called vertex and edge stalks, to each vertex $v\in V$ and edge $e\in E$, respectively, and, for every directed edge $e=(u\to v)$, linear source and target maps
\begin{equation*}
  S_e:\mathcal F(u)\to\mathcal F(e),
  \qquad
  T_e:\mathcal F(v)\to\mathcal F(e).
\end{equation*}
Vertex assignments form the $0$-cochain space $C^0(G;\mathcal F):=\bigoplus_{v\in V}\mathcal F(v)$, whilst edge assignments form the $1$-cochain space $C^1(G;\mathcal F):=\bigoplus_{e\in E}\mathcal F(e)$.
Writing $s(e)$ and $t(e)$ for the source and target vertices, respectively, the coboundary $D:C^0(G;\mathcal F)\to C^1(G;\mathcal F)$ gives the induced edge discrepancies,
\begin{equation*}
  (Dx)_e=T_ex_{t(e)}-S_ex_{s(e)}.
\end{equation*}
\end{definition}

\section{Temporal Predictive Coding as Directed Sheaf Diffusion}
\label{sec:directed-sheaf-diffusion}
\citet{seely2025sheaf} recasts linear predictive coding networks in terms of cellular sheaves. We extend this framework to tPC networks by constructing a directed network sheaf over the network's computational graph.

We assign vertex and edge stalks $\mathcal F(v_{\ell, t})=\mathcal F(e^s_{\ell,t})=\mathcal F(e^\tau_{\ell,t})=H_\ell$, where vertices and edges are denoted by \(v_{\ell,t}\) and \(e^\bullet_{\ell,t}\), respectively, and choose source and target maps so that $D$ recovers the spatial and temporal edge residuals. 
{\small
\begin{equation*}
\begin{aligned}
e^s_{\ell,t}:(\ell+1,t)\to(\ell,t):\qquad&
(S_e,T_e)=(W_\ell,I),
&
(Dz)_e
=z_{\ell,t}-W_\ell z_{\ell+1,t}
=r^s_{\ell,t},\\
e^\tau_{\ell,t}:(\ell,t)\to(\ell,t+1):\qquad&
(S_e,T_e)=(A_\ell,I),
&
(Dz)_e
=z_{\ell,t+1}-A_\ell z_{\ell,t}
=r^\tau_{\ell,t}.
\end{aligned}
\end{equation*}
}

In non-temporal PC, spatial residuals act on both incident activities. Temporal PC is instead causal: a temporal residual on $t\to t+1$ may update its receiving state at $t+1$, but cannot modify the already inferred state at $t$. We encode this with a response operator $R:C^0(G;\mathcal F)\to C^1(G;\mathcal F)$. Spatial edges respond reciprocally, whereas temporal edges are receiver-only:
\begin{equation*}
  D_{e^s_{\ell,t}}
  =R_{e^s_{\ell,t}}
  =[-W_\ell\;\;I],
  \qquad
  D_{e^\tau_{\ell,t}}
  =[-A_\ell\;\;I],
  \qquad
  R_{e^\tau_{\ell,t}}
  =[0\;\;I].
\end{equation*}
Thus $D$ determines the residual, whilst $R^\ast$ determines where that residual acts.\footnote{Throughout, adjoints are taken with respect to the standard inner products on the stalks, with precision operators written explicitly.} Writing $D_s$ for the spatial coboundary and $D_\tau,R_\tau$ for the temporal operators, one obtains the directed sheaf Laplacian
$$\mathcal{L}_{\mathrm{tot}}( z) =  -D_s^\ast\Sigma D_sz
  -R_\tau^\ast\Lambda D_\tau z.$$
Here $\Sigma\succ0$ and $\Lambda\succ0$ are symmetric block-diagonal precision operators. Making temporal responses reciprocal by setting $R_\tau=D_\tau$ recovers ordinary weighted sheaf diffusion.

\section{Spatial Inference Filters Temporal Recurrence}
\label{sec:spatial-inference}

We now consider inference at a fixed time step $t$. Clamping the bottom-layer activity to the current observation $z_{0,t}=o_t$ and collecting the free activations into
\begin{equation*}
  z_t:=(z_{1,t},\ldots,z_{L,t})^\top\in H_{ f},
  \qquad
  H_f:=\bigoplus_{\ell=1}^{L}H_\ell,
\end{equation*}
we can write the stacked spatial residual in affine form as
\begin{equation*}
  r_t^s
  =\delta z_t-b_t
  =\bigl(o_t-W_0z_{1,t},\,
    z_{1,t}-W_1z_{2,t},\,\ldots,\,
    z_{L-1,t}-W_{L-1}z_{L,t}\bigr)^\top,
  \qquad
  b_t:=(-o_t,0,\ldots,0)^\top,
\end{equation*}
where $\mathcal{E}_t^s:=H_0\oplus\cdots\oplus H_{L-1}$ denotes the spatial edge space at time $t$, and $\delta:H_f\to \mathcal{E}_t^s$ denotes the spatial coboundary restricted to the free vertices, with the clamped observation absorbed into $b_t$.
On these coordinates, define the temporal map $A$ and its incoming proposal
$a_t$ by
\begin{equation*}
  A:=\bigoplus_{\ell=1}^{L}A_\ell:
  H_f \longrightarrow H_f,
  \qquad
  a_t:=Az^{\star}_{t-1}.
\end{equation*}
Accordingly, inference at time $t$ minimises the slice energy
\begin{equation*}
  E_t(z_t; a_t, b_t) = \frac{1}{2}\lVert \delta z_t - b_t \rVert^2_{\Sigma} + \frac{1}{2}\lVert z_t - a_t \rVert_{\Lambda}^2.
\end{equation*}
Writing \(\mathcal L=\delta^\ast\Sigma\delta\) for the \emph{spatial sheaf Laplacian}, let
\begin{equation*}
  y_t:=\mathcal L^\dagger\delta^\ast\Sigma b_t,
  \qquad
  \mathcal L y_t=\delta^\ast\Sigma b_t,
\end{equation*}
where $\mathcal L^\dagger$ denotes the Moore--Penrose pseudoinverse.

In this clamped setting, directed sheaf diffusion restricts to the free activities, and yields the corresponding relative flow
\begin{equation*}
  \dot z_t=-\delta^\ast\Sigma(\delta z_t-b_t)-\Lambda(z_t-a_t)=-\mathcal L(z_t-y_t)-\Lambda(z_t-a_t).
\end{equation*}
Since $\mathcal L+\Lambda\succ0$, this flow has a unique inference equilibrium. Define
\begin{equation*}
  P=(\mathcal L+\Lambda)^{-1}\Lambda,
  \qquad
  Q=(\mathcal L+\Lambda)^{-1}\mathcal L,
  \qquad
  P+Q=I.
\end{equation*}
The inference equilibrium is
\begin{equation}\label{eq: zstar}
  \boxed{z_t^\star=Pa_t+Qy_t \, . }
\end{equation}
We call $P$ the \emph{persistence operator} and $Q$ the \emph{exposure operator}. Equation~\eqref{eq: zstar} shows that inference mediates between the temporal proposal and spatial optimum, retaining one component in the state and exposing the complementary component as temporal error. Furthermore, the irreducible tension between spatial and temporal error is captured by the frustration $\kappa_t := \delta a_t - b_t$. A sheaf-theoretic derivation of the inference equilibrium and inference maps, as well as a cohomological interpretation of the frustration $\kappa_t$ is provided in Appendix~\ref{apx: sheaf derivation}.

Writing $m_t = a_t - y_t$, we obtain
\begin{equation*}
  z_t^\star-y_t=Pm_t,
  \qquad
  a_t-z_t^\star=Qm_t.
\end{equation*}
Finally, substituting $a_t=Az_{t-1}^\star$ and holding the current observation fixed gives
\begin{equation*}
  z_t^\star=PAz_{t-1}^\star+Qy_t.
\end{equation*}
The learned recurrent map is $A$, whilst one-step state influence propagates through the inference-corrected transition $F:=PA$.

\section{Spatial Modes Balance Competing Predictions}
\label{sec:spatial-modes}

In this section, we suppose the temporal precision is isotropic: $\Lambda = \lambda I$, where $\lambda \in \mathbb{R}_{>0}$. We recall that the spatial sheaf Laplacian $\mathcal L = \delta^\ast\Sigma\delta$ is self-adjoint and positive semidefinite. It thus admits an orthonormal eigendecomposition
\begin{equation*}
  \mathcal L = \sum_i \ell_i \Pi_i, \quad \ell_i \geq 0,
\end{equation*}
where $\Pi_i$ is the orthogonal projection onto the eigenspace of $\ell_i$. The inference operators become
\begin{equation*}
P=\lambda(\mathcal L+\lambda I)^{-1}, \qquad Q=\mathcal L(\mathcal L+\lambda I)^{-1}=I-P.
\end{equation*}
Notably, $P$ and $Q$ preserve the eigenspaces of $\mathcal{L}$.

\begin{proposition}[Mode law]
Let $v_i$ be a simultaneous eigenvalue of $\mathcal{L},Q$, and $P$, with corresponding eigenvalues $\ell_i, q_i, p_i$. Let $\mu_i := \ell_i/\lambda$. Then
\begin{equation*}
  p_i=\frac{1}{1+\mu_i},
  \qquad
  q_i=\frac{\mu_i}{1+\mu_i},
  \qquad
  p_i+q_i=1.
\end{equation*}
\end{proposition}
Suppose we perturb the free hidden activities in the current slice by $v_i$. Since
\begin{equation*}
  \langle v_i,\mathcal L v_i\rangle=\|\delta v_i\|_{\Sigma}^2,
\end{equation*}
its eigenvalue $\ell_i$ measures the spatial penalty induced by the cross-layer perturbation, whilst $\mu_i=\ell_i/\lambda$ compares this penalty to the temporal one. Accordingly, low-$\mu_i$ modes are preferentially retained by inference, whereas high-$\mu_i$ modes are preferentially exposed as temporal error.

Since \(p_i+q_i=1\), inference resolves the spatial-temporal disagreement modewise: \(p_i\) gives the fraction by which the inferred state moves from the spatial solution towards the temporal proposal, whilst \(q_i\) gives the fraction left as temporal error. When recurrence places a proposal component in mode $i$, $p_i$ is its persistence gain and $q_i$ its exposure gain. Although this modal discussion assumes $\Lambda=\lambda I$, the following identity holds for arbitrary $\Lambda\succ0$.

\begin{theorem}[One-slice persistence--error identity]
For any $x\in H_{f}$,
\begin{equation*}
\|x\|_\Lambda^2
=
\|Px\|_\Lambda^2
+
\|Qx\|_\Lambda^2
+
2\|\delta Px\|_\Sigma^2.
\end{equation*}
In particular, for the inference mismatch $m=a-y$,
\begin{equation*}
\|m\|_\Lambda^2
=
\underbrace{\|Pm\|_\Lambda^2}_{\text{accepted displacement}}
+
\underbrace{\|Qm\|_\Lambda^2}_{\text{temporal error}}
+
\underbrace{2\|\delta Pm\|_\Sigma^2}_{\text{induced spatial interference}}.
\end{equation*}
\end{theorem}
The final term appears since $Pm$ and $Qm$ are not generally $\Lambda$-orthogonal. Here, $Pm = z^\star - y$ is the displacement away from the spatial optimum accepted by the temporal proposal, and therefore $\delta Pm=\delta(z^\star-y)$ is the additional spatial residual induced by moving away from that optimum. Thus, the identity accounts for the mismatch through state displacement, proposal-side temporal error, and the additional spatial residual induced by that displacement. For a pure mode $v_i$, the corresponding shares are $p_i^2$, $q_i^2$, and $2p_iq_i$, with maximal spatial interference at $\mu_i=1$.

\section{What Local Temporal Learning Omits}
\label{sec:temporal-credit}

We now turn our attention to learning. The standard tPC temporal update is temporally local: it cannot determine whether a locally favourable change will impair or improve predictions several steps later \citep{ngkeekwong2026complex}. Since the inferred state at time $t$ causally predicts states at future time steps, we can quantify this omitted dependence by seeing how a perturbation of the current state affects future scalar objectives, such as later slice energies or a downstream task loss.

Let $\widetilde E_t(a_t,b_t):=\min_{z_t}E_t(z_t;a_t,b_t)$ denote the
reduced current-slice energy. With parameters held fixed, let $\Phi_t$ denote
any future scalar objective whose dependence on the proposal $a_t$ is through
the inferred state $z_t^\star$ and subsequent states, and write
$g_t:=\nabla_{z_t^\star}\Phi_t$ for its sensitivity to the current inferred
state. Future sensitivity is transported backward through
the adjoints of the inference-corrected transition $F=PA$. The current
temporal proposal $a_t$ affects this future objective through $z_t^\star$,
whilst $\widetilde E_t$ provides the local temporal signal. Combining these
contributions gives the exact proposal-level credit.

\begin{theorem}[Temporal credit split]
\begin{equation*}
    \nabla_{a_t}\bigl(\widetilde E_t+\Phi_t\bigr)
    =
    \underbrace{\Lambda Q(a_t-y_t)}
        _{\text{local temporal-error signal}}
    +
    \underbrace{P^\ast g_t}
        _{\text{future credit}}.
\end{equation*}
\end{theorem}
Since $a_t=Az_{t-1}^\star$, this proposal-level decomposition induces the
corresponding temporal-weight updates by the chain rule; the explicit
layerwise gradients are given in
\hyperref[app:temporal-weight-gradients]{Appendix~\ref*{app:temporal-weight-gradients}}.
The standard tPC temporal update retains only the
first term, updating the temporal weights from the temporal error available
at the current slice. The contribution lost is modulated by the adjoint of
$P$. In the modal case considered in the previous section, $P$ is
self-adjoint: $q_i$ is the exposure gain on the local temporal-error signal, whilst
$p_i$ is the persistence gain on returned future credit.

 \section{Long-Term Persistence and Mode Routing}
\label{sec:long-term-routing}

We now discuss how perturbations propagate through the network over long time
horizons. We begin with a perturbation
$x_{t-1}=\Delta z_{t-1}^{\star}$ to the inferred state at time $t-1$, whilst
keeping the input sequence fixed. At time $t$, recurrence gives a proposal
perturbation $\Delta a_t=Ax_{t-1}$. Since the input sequence is fixed, the
spatial optimum $y_t$ remains unchanged, and the inference split gives
\begin{equation*}
    x_t
    :=\Delta z_t^\star
    =P\Delta a_t
    =PAx_{t-1}
    =Fx_{t-1},
    \qquad
    e_t
    :=-(x_t-\Delta a_t)
    =Q\Delta a_t
    =QAx_{t-1}.
\end{equation*}
Here $x_t$ is the part of the perturbation carried into the next inferred
state, whilst $e_t$ is its proposal-side exposure at that slice.
At each transition, $A$ mixes the current state across spatial modes;
inference then retains the $P$-weighted component in the next state and
exposes the $Q$-weighted component as temporal error. These transitions are
between exact slice equilibria, rather than inference iterations within a slice.
Iterating the state response, a perturbation $x_t$ contributes after
$n\geq1$ further steps
\begin{equation*}
    x_{t+n}=F^n x_t,
    \qquad
    e_{t+n}=QAF^{n-1}x_t
    =QA(PA)^{n-1}x_t.
\end{equation*}

\begin{figure}[!t]
  \centering
  \input{figures/paper_native_credit_comparison}
  \caption{Full temporal credit learns to retain the cue in $C$ and expose it
  in $E$ at the response, whereas standard tPC does not.}
  \label{fig:delayed-recall-routing}
\end{figure}

\subsection{A Single-Mode Obstruction}
We now return to the isotropic case and consider what occurs when the perturbation remains in a single spatial mode. Consider the special single invariant-mode case where $v_i$ is additionally a fixed point of the temporal operator $A$, and $x_t$ is a perturbation in the spatial mode $v_i$. The perturbation is retained in the next state with gain $p_i$, whilst its induced temporal error is scaled by $q_i$. Therefore, after $n\geq1$ time steps,
\begin{equation*}
    x_{t+n}=p_i^n x_t,
    \qquad
    e_{t+n}=q_ip_i^{\,n-1}x_t.
\end{equation*}
Thus $p_i^n$ measures persistence of the earlier influence in the state,
whereas $q_ip_i^{n-1}$ measures its temporal-error report after $n$ time steps. Using $p_i=1/(1+\mu_i)$ and $q_i=\mu_i/(1+\mu_i)$, the delayed exposure gain is
\begin{equation*}
    q_ip_i^{\,n-1}
    =
    \frac{\mu_i}{(1+\mu_i)^n}.
\end{equation*}
For $n>1$, this is maximised at
\begin{equation*}
    \mu_i^\star=\frac{1}{n-1},
    \qquad
    \max_{\mu_i\geq0}q_ip_i^{\,n-1}
    =
    \frac1n
    \left(\frac{n-1}{n}\right)^{n-1}
    \sim\frac{1}{en}.
\end{equation*}
Therefore, even in the optimal case, the single-mode delayed exposure gain decays as $O(1/n)$. In particular, a perturbation in a single invariant mode cannot both persist strongly through inference steps and produce a strong delayed temporal-error report.

\subsection{Mode Routing}
The preceding section assumes that perturbations persist in a single mode. In general, however, this is untrue. Although $A$ is shared across time, its input state evolves under $PA$; temporal position can therefore be encoded in the state trajectory itself, so the same recurrence can produce weak intermediate temporal errors and a strong later report. Writing
\begin{equation*}
    P=\sum_i p_i\Pi_i,
    \qquad
    Q=\sum_i q_i\Pi_i,
    \qquad
    A_{ij}:=\Pi_iA\Pi_j,
\end{equation*}
the block $A_{ij}$ maps perturbations between eigenspaces $j$ and $i$ of the spatial Laplacian $\mathcal L$. For a perturbation $x_t$ and delay $n\geq1$, its temporal-error contribution at time $t+n$ is $e_{t+n}=QA(PA)^{n-1}x_t$. If there are $M$ spatial modes, expanding this propagation operator gives
\begin{equation*}
    QA(PA)^{n-1}
    =
    \sum_{i_0,\ldots,i_n\in\{1,\ldots,M\}}
    q_{i_n}
    \left(
        \prod_{k=1}^{n-1}p_{i_k}
    \right)
    A_{i_ni_{n-1}}\cdots A_{i_1i_0}.
\end{equation*}
Thus each route $i_0 \to \cdots \to i_n$ contributes a product of temporal mode transitions $A_{i_ni_{n-1}}\cdots A_{i_1i_0}$, weighted by the persistence gains of
the intermediate modes and the exposure gain of the final mode. The temporal map therefore determines how perturbations move between spatial modes, whilst inference determines how they persist or appear as temporal errors. As the expansion is a sum, different routes may reinforce or cancel.

Recurrence can thus evade the single-mode obstruction. A perturbation may travel through low-$\mu_i$ modes, for which $p_i \approx 1$ over the intermediate slices, and then be routed into a high-$\mu_i$ mode, for which $q_i \approx 1$ at the final time step, yielding an inference gain $q_{i_n}\prod_{k=1}^{n-1}p_{i_k}\approx1$.

\section{Learning Mode Routing in Delayed Recall}
\label{sec:delayed-recall-example}
We consider whether the mode-routing mechanism described in Section~\ref{sec:long-term-routing} can be learned in a real tPC network. We construct a one-hidden-layer linear tPC network with a fixed spatial Laplacian
\begin{equation*}
  \mathcal L = V \operatorname{diag}(\varepsilon, \dots, \varepsilon, \varepsilon^{-1}) V^{\top} \in \mathbb{R}^{N \times N},
  \qquad N=16, \quad \varepsilon=0.01.
\end{equation*}
where $V\in O(N)$ is drawn independently and fixed for each seed.  With this choice of $\varepsilon$, the first $N-1$ modes form a carrier subspace $C = \operatorname{span}\{v_1, \dots, v_{N-1}\}$ with persistence $p_C = (1+\varepsilon)^{-1} \approx 1$, with the final mode forming an exposure subspace $E = \operatorname{span}\{v_N\}$ with exposure $q_E = (1 + \varepsilon)^{-1} \approx 1$. We fix a corresponding $W$, and only $A$ is learned. At time 0, we inject a scalar cue $c \sim \mathcal{N}(0, 1)$ along the observation direction for $v_1$, followed by four Gaussian distractor slices at times $t=1,\dots,4$ using directions $v_2,\dots,v_5$. At response time $R=5$, we clamp a bottom-layer response target chosen to drive the hidden state towards $cv_N$, encouraging exposure of the cue. During inference, we sequentially clamp the cue, Gaussian distractors, and the response target while minimising the slice energy at each step. During learning, we use the same inference procedure and slice energies, but update $A$ two ways. In the standard tPC setting (b), we update $A$ using each slice's temporal error. By contrast, in the full credit setup (a) we update $A$ by differentiating the full predictive coding energy through the entire inferred trajectory. A single optimisation step consists of inferring the entire trajectory for every sequence in a minibatch and taking one Adam update of $A$. During testing, we draw a new cue and distractors, and infer sequentially. At $t = R$, we set the response observation to zero, infer the final state from the temporal proposal, and read the predicted cue from the exposed component of the final temporal error.

Results are shown in Figure~\ref{fig:delayed-recall-routing}. Across all 15 seeds, full temporal credit learned to retain--then--expose the cue, with median test MSE falling to $0.00548$, compared with the unit-error no-information baseline. By contrast, standard tPC failed to do so, with median test MSE remaining near baseline at $0.97057$ and cue-dependent carrier energy collapsing during training. This highlights an important limitation of standard tPC learning: while mode routing can overcome the single-mode obstruction, its temporally local update lacks the long-range credit required to learn this route in our experiment. Further details are provided in Appendix~\ref{app:paper-native-tpc}. On overlapping 30-digit copy, full credit learns the same route (median test MSE $0.011$ versus $0.999$ for local tPC; Appendix~\ref{app:copy-task}).

\section{Discussion}

Temporal predictive coding networks, while attractive for their layer-local,
temporally local learning rules, can struggle to learn long-range dependencies
\citep{potter2026longrange, ngkeekwong2026complex}.  This work introduces the persistence and exposure
operators $P$ and $Q$, which offer a new interpretable framework for
understanding how interactions between spatial and temporal errors shape credit
propagation over long horizons.  We show that a linear tPC network's recurrent map may
contain dynamics capable of both preserving and routing information.  However,
standard local tPC omits the future-credit term $P^\ast g_t$ captured by full
temporal credit and may therefore fail to learn weights that realise this
capacity.

More broadly, the $P,Q$ framework makes the role of spatiotemporal tension explicit
in temporal computation, highlighting its impact on state persistence, error
exposure, effective recurrence, and future credit for linear networks. This opens up several avenues
for future work, including nonlinear extensions and geometry-aware
architectures for long-range learning in tPC.

\bibliography{refs}

\appendix

\section{Proofs for Spatial Inference}

This appendix assumes finite-dimensional stalks and cochain spaces.

\subsection{Directed Slice Dynamics}

Spatial edges respond reciprocally, so $R_s^\ast=D_s^\ast$. Their free-state contribution is $-\delta^\ast\Sigma(\delta z_t-b_t)$. For a temporal residual entering slice $t$, the target block of $R_\tau$ is the identity, giving $-\Lambda(z_t-a_t)$ at the receiver. The source block of an outgoing temporal edge is zero, so that residual contributes no force to slice $t$. Hence
\begin{equation*}
  \dot z_t
  =-\delta^\ast\Sigma(\delta z_t-b_t)
  -\Lambda(z_t-a_t),
\end{equation*}
which depends only on slice $t$ and the incoming proposal $a_t=Az_{t-1}^\star$.

\subsection{Spatial Residual Decomposition}
\label{app:spatial-residual-decomposition}

Let $s_t=\delta^\ast\Sigma b_t$. Since
\begin{equation*}
  \langle x,\mathcal Lx\rangle
  =\lVert\delta x\rVert_{\Sigma}^2,
\end{equation*}
we have $\ker\mathcal L=\ker\delta$. Hence, for every $x\in\ker\mathcal L$,
\begin{equation*}
  \langle x,s_t\rangle
  =\langle\delta x,\Sigma b_t\rangle=0.
\end{equation*}
Because $\mathcal L$ is self-adjoint,
$(\ker\mathcal L)^\perp=\operatorname{im}\mathcal L$, so
$s_t\in\operatorname{im}\mathcal L$. It follows that
$\mathcal L\mathcal L^\dagger s_t=s_t$, so
$y_t=\mathcal L^\dagger s_t$ satisfies the normal equations.

Differentiating the spatial least-squares term at $y_t$ gives
\begin{equation*}
  0
  =\delta^\ast\Sigma(\delta y_t-b_t)
  =\mathcal L y_t-\delta^\ast\Sigma b_t.
\end{equation*}
For $\eta_t=b_t-\delta y_t$, this normal equation implies
\begin{equation*}
  \delta^\ast\Sigma \eta_t
  =\delta^\ast\Sigma(b_t-\delta y_t)=0.
\end{equation*}
The cross term in
\begin{equation*}
  \lVert\delta z_t-b_t\rVert_{\Sigma}^2
  =\lVert\delta(z_t-y_t)-\eta_t\rVert_{\Sigma}^2
\end{equation*}
therefore vanishes, leaving $\lVert \eta_t\rVert_{\Sigma}^2$ as an additive constant and proving the reduced slice energy.

\subsection{Inference Operators}

Since $\mathcal L\succeq0$ and $\Lambda\succ0$,
$\mathcal L+\Lambda\succ0$ is invertible. The equilibrium equation is equivalent to
\begin{equation*}
  (\mathcal L+\Lambda)z_t^\star=\Lambda a_t+\mathcal L y_t,
\end{equation*}
which yields $z_t^\star=Pa_t+Qy_t$ and
\begin{equation*}
  P+Q=(\mathcal L+\Lambda)^{-1}(\Lambda+\mathcal L)=I.
\end{equation*}
Consequently, for $m_t=a_t-y_t$,
\begin{equation*}
  z_t^\star-y_t
  =P(a_t-y_t)=Pm_t,
  \qquad
  a_t-z_t^\star
  =Q(a_t-y_t)=Qm_t.
\end{equation*}

For force balance,
\begin{equation*}
  \mathcal L P
  =
  \bigl((\mathcal L+\Lambda)-\Lambda\bigr)
  (\mathcal L+\Lambda)^{-1}\Lambda
  =
  \Lambda-\Lambda(\mathcal L+\Lambda)^{-1}\Lambda
  =
  \Lambda Q,
\end{equation*}
so $\mathcal L Pm_t=\Lambda Qm_t$ for every mismatch $m_t$. This is the equality of the spatial force $\mathcal L(z_t^\star-y_t)$ and the opposing temporal force $\Lambda(a_t-z_t^\star)$.

Finally, with the current observation and weights fixed, $y_t$, $P$, $Q$, and $A$ are independent of $z_{t-1}^\star$. Differentiating
\begin{equation*}
  z_t^\star=PAz_{t-1}^\star+Qy_t
\end{equation*}
therefore gives $\partial z_t^\star/\partial z_{t-1}^\star=PA=F$.

\subsection{Mode Law}

Since $\Sigma\succ0$, the spatial sheaf Laplacian satisfies
\begin{equation*}
  \mathcal L^\ast=\mathcal L,
  \qquad
  \langle x,\mathcal Lx\rangle
  =\langle\delta x,\Sigma\delta x\rangle
  =\lVert\delta x\rVert_\Sigma^2\geq0.
\end{equation*}
Thus the spectral theorem gives an orthogonal decomposition into
eigenspaces of $\mathcal L$. Under the isotropic temporal precision
$\Lambda=\lambda I$, let $\mathcal L v_i=\ell_iv_i$ and
$\mu_i=\ell_i/\lambda$. Applying
$P=\lambda(\mathcal L+\lambda I)^{-1}$ to
$v_i$ yields
\begin{equation*}
  Pv_i
  =\frac{\lambda}{\ell_i+\lambda}v_i
  =\frac{1}{1+\mu_i}v_i.
\end{equation*}
Likewise,
\begin{equation*}
  Qv_i
  =\frac{\ell_i}{\ell_i+\lambda}v_i
  =\frac{\mu_i}{1+\mu_i}v_i.
\end{equation*}
Therefore $p_i=(1+\mu_i)^{-1}$,
$q_i=\mu_i(1+\mu_i)^{-1}$, and $p_i+q_i=1$, proving the mode law.

\subsection{One-Slice Persistence--Error Identity}
\label{app:one-slice-inference-identity}

From $P+Q=I$, every $x\in H_{\mathrm f}$ satisfies $x=Px+Qx$. Expanding its
squared $\Lambda$-norm gives
\begin{equation*}
  \lVert x\rVert_\Lambda^2
  =\lVert Px\rVert_\Lambda^2
  +\lVert Qx\rVert_\Lambda^2
  +2\langle Px,Qx\rangle_\Lambda.
\end{equation*}
The force-balance identity $\mathcal L P=\Lambda Q$ implies
\begin{align*}
  \langle Px,Qx\rangle_\Lambda
  &=\langle Px,\Lambda Qx\rangle \\
  &=\langle Px,\mathcal L Px\rangle \\
  &=\langle\delta Px,\Sigma\delta Px\rangle
    =\lVert\delta Px\rVert_\Sigma^2.
\end{align*}
Substitution proves
\begin{equation*}
  \lVert x\rVert_\Lambda^2
  =\lVert Px\rVert_\Lambda^2
  +\lVert Qx\rVert_\Lambda^2
  +2\lVert\delta Px\rVert_\Sigma^2.
\end{equation*}

In the isotropic case, the identity can be read modewise. Let
$\Lambda=\lambda I$ and $x=cv_i$. Dividing each term by the total mismatch
energy $\lVert x\rVert_\Lambda^2$ gives the persistence, exposure, and spatial-interference shares
\begin{equation*}
  p_i^2,
  \qquad
  q_i^2,
  \qquad
  2\mu_i p_i^2
  =2p_iq_i
  =\frac{2\mu_i}{(1+\mu_i)^2},
\end{equation*}
respectively. These shares sum to $(p_i+q_i)^2=1$. The spatial-interference
share has derivative
\begin{equation*}
  \frac{2(1-\mu_i)}{(1+\mu_i)^3}.
\end{equation*}
It therefore increases up to $\mu_i=1$ and decreases thereafter, attaining
its maximum value $1/2$ when $p_i=q_i=1/2$.

\subsection{General Temporal Precision}

The modal result extends to every $\Lambda\succ0$ after changing to the
$\Lambda$-metric.  Define
\begin{equation*}
  \widehat{\mathcal L}
  =\Lambda^{-1/2}\mathcal L\Lambda^{-1/2}.
\end{equation*}
This operator is self-adjoint and positive semidefinite.  Since
\begin{equation*}
  \mathcal L+\Lambda
  =\Lambda^{1/2}(I+\widehat{\mathcal L})\Lambda^{1/2},
\end{equation*}
the inference operators satisfy
\begin{equation*}
  \Lambda^{1/2}P\Lambda^{-1/2}
  =(I+\widehat{\mathcal L})^{-1},
  \qquad
  \Lambda^{1/2}Q\Lambda^{-1/2}
  =\widehat{\mathcal L}(I+\widehat{\mathcal L})^{-1}.
\end{equation*}
Let $u_i$ be an orthonormal eigenvector of $\widehat{\mathcal L}$ with eigenvalue
$\mu_i\geq0$, and put $v_i=\Lambda^{-1/2}u_i$.  Then
\begin{equation*}
  \mathcal L v_i=\mu_i\Lambda v_i,
  \qquad
  \langle v_i,\Lambda v_j\rangle=\delta_{ij}.
\end{equation*}
Applying the two similarity-transformed inference operators to $u_i$
and transforming back gives
\begin{equation*}
  Pv_i=\frac{1}{1+\mu_i}v_i,
  \qquad
  Qv_i=\frac{\mu_i}{1+\mu_i}v_i.
\end{equation*}
Thus the persistence and exposure gains retain the same form for general
temporal precision; only the diagonalising basis and inner product change.

\section{Proofs for Temporal Credit}

\subsection{Multi-Step Influence and Backward Transport}

Fix the observations and weights over the interval from $t$ to $t+n$.
For each $s>t$, the inferred dynamics are
\begin{equation*}
  z_s^\star=Fz_{s-1}^\star+Qy_s.
\end{equation*}
The second term is independent of $z_{s-1}^\star$, so the one-step
Jacobian is
\begin{equation*}
  \frac{\partial z_s^\star}{\partial z_{s-1}^\star}=F.
\end{equation*}
For $n=1$, this gives the claimed expression directly.  If the claim
holds for $n-1$, the chain rule gives
\begin{align*}
  \frac{\partial z_{t+n}^\star}{\partial z_t^\star}
  &=
  \frac{\partial z_{t+n}^\star}{\partial z_{t+n-1}^\star}
  \frac{\partial z_{t+n-1}^\star}{\partial z_t^\star} \\
  &=F^n.
\end{align*}
Induction therefore proves the multi-step influence formula for every
$n\geq1$.

Now let a scalar contribution $\phi_{t+n}$ depend on the current state
only through $z_{t+n}^\star$, and write
$g_{t+n}=\nabla_{z_{t+n}^\star}\phi_{t+n}$.  For a perturbation
$d z_t^\star$,
\begin{align*}
  d\phi_{t+n}
  &=\left\langle
      g_{t+n},
      \frac{\partial z_{t+n}^\star}{\partial z_t^\star}
      d z_t^\star
    \right\rangle \\
  &=\left\langle
      \left(
        \frac{\partial z_{t+n}^\star}{\partial z_t^\star}
      \right)^\ast g_{t+n},
      d z_t^\star
    \right\rangle.
\end{align*}
Hence
\begin{equation*}
  g_t
  =(F^n)^\ast g_{t+n}
  =(F^\ast)^n g_{t+n}.
\end{equation*}
If the future objective is a sum of contributions at several times, its
gradient at $z_t^\star$ is the corresponding sum of these adjoint
pullbacks.

\subsection{Reduced Slice Energy and Proposal Sensitivity}

Suppress the time index and let $m=a-y$.  At equilibrium,
\begin{equation*}
  z^\star-y=Pm,
  \qquad
  z^\star-a=-Qm.
\end{equation*}
Using the spatial residual decomposition from
\hyperref[app:spatial-residual-decomposition]{Appendix~\ref*{app:spatial-residual-decomposition}}, the reduced
slice energy is
\begin{equation*}
  \widetilde E(a,b)
  =\frac{1}{2}\lVert h\rVert_\Sigma^2
   +\frac{1}{2}\lVert\delta Pm\rVert_\Sigma^2
   +\frac{1}{2}\lVert Qm\rVert_\Lambda^2.
\end{equation*}
The force-balance identity gives
\begin{equation*}
  \langle Pm,\Lambda Qm\rangle
  =\langle Pm,\mathcal L Pm\rangle
  =\lVert\delta Pm\rVert_\Sigma^2.
\end{equation*}
Since $m=Pm+Qm$, it follows that
\begin{equation*}
  \widetilde E(a,b)
  =\frac{1}{2}\lVert h\rVert_\Sigma^2
   +\frac{1}{2}\langle m,\Lambda Qm\rangle.
\end{equation*}
Moreover,
\begin{equation*}
  \Lambda Q
  =\Lambda(\mathcal L+\Lambda)^{-1}\mathcal L
  =\Lambda-\Lambda(\mathcal L+\Lambda)^{-1}\Lambda,
\end{equation*}
which is self-adjoint because $\mathcal L+\Lambda$ is self-adjoint and
positive definite. Holding $b$ fixed makes $y$ and $h$ independent of $a$, so
$d m=d a$.  Differentiating the reduced energy therefore yields
\begin{equation*}
  \nabla_a\widetilde E
  =\Lambda Qm
  =\Lambda Q(a-y).
\end{equation*}
The equilibrium state is affine in the proposal,
\begin{equation*}
  z^\star=Pa+Qy.
\end{equation*}
With the observation and weights fixed, $P$, $Q$, and $y$ do not vary
with $a$.  Thus
\begin{equation*}
  d z^\star=P\,d a,
  \qquad
  \frac{\partial z^\star}{\partial a}=P.
\end{equation*}

\subsection{Temporal Credit Split}

Restore the time index and let
$g_t=\nabla_{z_t^\star}\Phi_t$.  The chain rule and the proposal
sensitivity above give, for every perturbation $d a_t$,
\begin{align*}
  d\Phi_t
  &=\langle g_t,d z_t^\star\rangle \\
  &=\langle g_t,P\,d a_t\rangle \\
  &=\langle P^\ast g_t,d a_t\rangle.
\end{align*}
Consequently,
\begin{equation*}
  \nabla_{a_t}\Phi_t=P^\ast g_t.
\end{equation*}
Adding this derivative to the reduced current-slice derivative proves
the temporal credit split:
\begin{equation*}
  \nabla_{a_t}\bigl(\widetilde E_t+\Phi_t\bigr)
  =\Lambda Q(a_t-y_t)+P^\ast g_t.
\end{equation*}

\subsection{Temporal-Weight Gradient and Modal Gains}
\label{app:temporal-weight-gradients}

Let
\begin{equation*}
  c_t
  :=\Lambda Q(a_t-y_t)+P^\ast g_t
  =\nabla_{a_t}\bigl(\widetilde E_t+\Phi_t\bigr).
\end{equation*}
Since $a_t=Az_{t-1}^\star$ and
$A=\bigoplus_{\ell=1}^L A_\ell$, writing $[u]_\ell$ for the layer-$\ell$
component of $u$ gives the per-transition layerwise chain rule
\begin{equation*}
  \nabla_{A_\ell}^{(t)}
  =\left[
      \Lambda Q(a_t-y_t)+P^\ast g_t
    \right]_\ell
    (z_{\ell,t-1}^\star)^\ast,
  \qquad \ell=1,\ldots,L.
\end{equation*}

Finally, in the isotropic case $\Lambda=\lambda I$, both $P$ and $Q$ are
self-adjoint functions of $\mathcal L$. With
$m_i=\Pi_i(a-y)$ and $g_i=\Pi_i g$, projection onto spatial mode $i$
gives
\begin{equation*}
  \Pi_i\Lambda Q(a-y)=\lambda q_i m_i,
  \qquad
  \Pi_iP^\ast g=p_i g_i.
\end{equation*}
Thus $q_i$ is the modewise exposure gain on the local temporal-error signal and $p_i$
is the modewise persistence gain on returned future credit. Their
identity $p_i+q_i=1$ is a statement about these operator gains; the
realised magnitudes can also depend on the distinct vectors $m_i$ and
$g_i$.

\section{Proofs for Long-Term Persistence and Mode Routing}
This appendix assumes finite-dimensional stalks and cochain spaces.

\subsection{Single-Mode Delayed Exposure}

Let $x_t=cv_i$ and assume $Av_i=v_i$.
The mode law gives $Pv_i=p_iv_i$ and $Qv_i=q_iv_i$.  Hence
\begin{equation*}
  x_{t+1}=PAx_t=p_ix_t,
  \qquad
  e_{t+1}=QAx_t=q_ix_t.
\end{equation*}
If $x_{t+n-1}=p_i^{n-1}x_t$, then another recurrent and inference step
gives
\begin{equation*}
  x_{t+n}=PAx_{t+n-1}=p_i^nx_t.
\end{equation*}
Induction proves the state formula for all $n\geq1$, and the report at
the final step is
\begin{equation*}
  e_{t+n}=QAx_{t+n-1}=q_ip_i^{n-1}x_t.
\end{equation*}

Using the mode law, define the delayed exposure gain
\begin{equation*}
  f_n(\mu)
  :=q(\mu)p(\mu)^{n-1}
  =\frac{\mu}{(1+\mu)^n},
  \qquad \mu\geq0.
\end{equation*}
For $n>1$,
\begin{equation*}
  f_n'(\mu)
  =\frac{1-(n-1)\mu}{(1+\mu)^{n+1}}.
\end{equation*}
The derivative is positive below $\mu^\star=(n-1)^{-1}$ and negative
above it, whilst $f_n(0)=0$ and $f_n(\mu)\to0$ as $\mu\to\infty$.
Therefore this critical point is the unique global maximiser, and
\begin{equation*}
  f_n(\mu^\star)
  =\frac{1}{n}\left(\frac{n-1}{n}\right)^{n-1}
  \sim\frac{1}{en}.
\end{equation*}
For $n=1$, $f_1(\mu)=\mu/(1+\mu)$ increases towards a supremum of one as
$\mu\to\infty$, so there is no finite maximiser. Thus the finite-maximiser
result applies only when $n>1$.

\subsection{Mode-Routing Path Expansion}

Let $\mathcal I=\{1,\ldots,M\}$ index the eigenspaces of $\mathcal L$. Insert
the spectral decompositions of $P$ and $Q$, together with
$I=\sum_{i\in\mathcal I}\Pi_i$, into $QA(PA)^{n-1}$.  This product has
$n$ factors of $A$, $n-1$ intermediate factors of $P$, and one final
factor of $Q$.  Expanding the products gives
\begin{align*}
  QA(PA)^{n-1}
  &=\sum_{i_0,\ldots,i_n\in\mathcal I}
    q_{i_n}
    \left(\prod_{k=1}^{n-1}p_{i_k}\right)
    \Pi_{i_n}A\Pi_{i_{n-1}}A\cdots
    \Pi_{i_1}A\Pi_{i_0} \\
  &=\sum_{i_0,\ldots,i_n\in\mathcal I}
    q_{i_n}
    \left(\prod_{k=1}^{n-1}p_{i_k}\right)
    A_{i_ni_{n-1}}\cdots A_{i_1i_0}.
\end{align*}
For $n=1$, the product of persistence gains is the empty product and is
equal to one.  Multiplying the operator identity by $x_t$ proves the
path expansion for $e_{t+n}$.

\section{Sheaf-Theoretic Derivations}
\label{apx: sheaf derivation}

\newcommand{\Em}[1]{\textbf{#1}}
\newcommand{\pois}{\mathcal{P}}
\def\N{\mathbb{N}}
\def\calF{\mathcal{F}}
\def\calH{\mathcal{H}}
\def\Z{\mathbb{Z}}
\def\Q{{\bf Q}}
\def\R{\mathbb{R}}
\def\C{\mathbb{C}}
\def\CC{{\cal C}}
\def\S{{\mathcal S}}
\def\Var{{\rm Var} \,}
\def\ee{\varepsilon}
\def\E{{\mathbb E}}
\def\cE{{\mathcal E}}
\def\P{{\mathbb P}}
\def\Cox{\hfill \Box}
\def\half{\frac{1}{2} \,}
\def\disp{\displaystyle}
\def\one{{\bf 1}}
\def\F{{\mathcal F}}
\def\calG{{\mathcal G}}
\def\GG{{\mathbb G}}
\def\calV{{\mathcal V}}
\def\calD{{\mathcal D}}
\def\calN{{\mathcal N}}
\def\calE{{\mathcal E}}
\def\B{{\mathcal B}}
\def\A{{\mathcal A}}
\def\xx{{\mathbf x}}
\def\yy{{\mathbf y}}
\def\ww{{\mathbf w}}
\def\zz{{\mathbf z}}
\def\g{{\mathfrak g}}
\def\h{{\mathfrak h}}
\def\d{{\mathtt d}}
\def\dalg{{\mathtt{d}_\mathrm{alg}}}
\def\AN{\aleph_0}
\def\cupdot{{\mathop{\dot{\cup}}}}

\def\dom{{\text{dom}}}
\def\zC{{C^0(\calG;\GG)}}
\def\oC{{C^1(\calG;\GG)}}
\def\zH{{H^0(\calG;\GG)}}
\def\oH{{H^1(\calG;\GG)}}
\def\ad{{\mathrm{ad}}}
\def\Ad{{\mathrm{Ad}}}
\def\cvx{{\mathrm{cvx}}}
\def\Valg{{V_{\mathrm{alg}}}}
\newcommand{\so}{\mathfrak{so}}
\newcommand{\gtwo}{\mathfrak{g}_2}
\newcommand{\Gtwo}{G_2}
\newcommand{\Hslice}[1]{\mathcal{H}^{ #1 }_{\mathrm{slice},t}}
\newcommand{\Eslice}{E_{\mathrm{slice}}}

\newcommand{\an}[1]{\| #1 \|_{\mathrm{alg}}}

\def\calL{{\mathcal L}}

In this appendix, we present an alternative derivation of the essential structure of linear temporal predictive coding networks. While the derivation in the main body is primarily focused on dynamics, the present appendix will instead take an sheaf-theoretic and Hodge-theoretic angle, focused on the relative cohomology of a sub-sheaf of full tPC network \citep{seely2025sheaf,curry2014sheaves}.

\subsection{The Time-Slice Sheaf}
We first consider a sub-network of a tPC network.

\begin{definition}
The \textbf{time-slice graph} $\calG^{(t)}$ at time $t \geq 1$ of a linear tPC network consists of the following data.

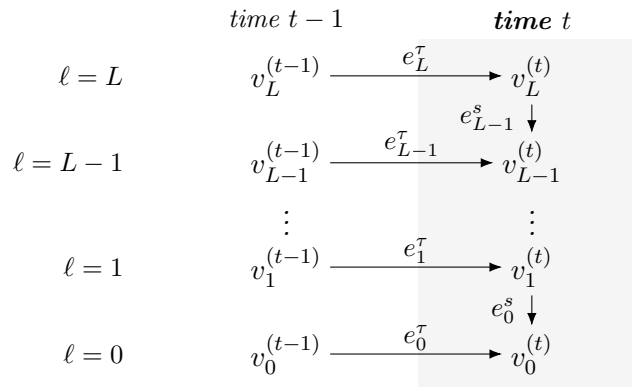
\begin{center}
\begin{tikzpicture}[
  x=3.25cm,
  y=1.15cm,
  >=Latex,
  activation/.style={inner sep=3pt},
  map/.style={->, thin},
  temporal label/.style={font=\small, inner sep=1.5pt},
  spatial label/.style={font=\small, inner sep=1pt,
    anchor=east, xshift=-5pt}
]
  \fill[black!4] (0.54,0.40) rectangle (1.46,4.42);

  \node[activation] (zm1L) at (0,4) {$v^{(t-1)}_L$};
  \node[activation] (ztL)  at (1,4) {$v^{(t)}_L$};

  \node[activation] (zm1Lm1) at (0,3) {$v^{(t-1)}_{L-1}$};
  \node[activation] (ztLm1)  at (1,3) {$v^{(t)}_{L-1}$};

  \node at (0,2.4) {$\vdots$};
  \node at (1,2.4) {$\vdots$};

  \node[activation] (zm11) at (0,1.8) {$v^{(t-1)}_1$};
  \node[activation] (zt1)  at (1,1.8) {$v^{(t)}_1$};

  \node[activation] (zm10) at (0,0.8) {$v^{(t-1)}_0$};
  \node[activation] (zt0)  at (1,0.8) {$v^{(t)}_0$};

  \draw[map] (ztL)  -- node[midway, spatial label] {$e^{s}_{L-1}$} (ztLm1);
  \draw[map] (zt1)  -- node[midway, spatial label] {$e^{s}_{0}$} (zt0);

  \draw[map] (zm1L) -- node[temporal label, above] {$e^{\tau}_{L}$} (ztL);
  \draw[map] (zm1Lm1) -- node[temporal label, above] {$e^\tau_{L-1}$} (ztLm1);
  \draw[map] (zm11) -- node[temporal label, above] {$e^\tau_1$} (zt1);
  \draw[map] (zm10) -- node[temporal label, above] {$e^\tau_0$} (zt0);

  \node[font=\small] at (0,4.65) {time $t-1$};
  \node[font=\small\bfseries] at (1,4.65) {time $t$};
  \node[font=\small, anchor=east] at (-0.62,4) {$\ell=L$};
  \node[font=\small, anchor=east] at (-0.62,3) {$\ell=L-1$};
  \node[font=\small, anchor=east] at (-0.62,1.8) {$\ell=1$};
  \node[font=\small, anchor=east] at (-0.62,0.8) {$\ell=0$};
\end{tikzpicture}
\end{center}

\end{definition}

\begin{remark}
    The time-slice graph contains all spatial edges at time-slice $t$, as well as the incoming temporal edges from time-slice $t-1$. Consequently, all relevant information for the time $t$ inference and learning steps are present.
\end{remark}

From the time-slice graph, we derive the time-slice sheaf.

\begin{definition}
    The \textbf{time-slice sheaf} $\mathcal{F}^{(t)} :  \mathcal{G}^{(t)} \to \mathbf{FinHilb}_{\R}$ at time $t \geq 1$ is the cellular sheaf consisting of the following data.
    \begin{enumerate}
        \item for each vertex $v_\ell^\bullet$, a vertex stalk $V_\ell^\bullet := V_\ell$, where $V_\ell$ is the activation space at layer $\ell$ of the tPC network. Each $V_\ell$ is endowed with an inner product, denoted by $\langle-,- \rangle$.
        \item For each temporal edge $v^{(t-1)}_\ell \xrightarrow{e^\tau_\ell} v^{(t)}_\ell$, $0\leq\ell\leq L$, an edge stalk $V^\tau_{\ell} := V_\ell$, with source restriction map $A_\ell: V_\ell^{(t-1)} \to V_\ell^{\tau}$ and target restriction map $I: V_\ell^{(t)} \to V_\ell^{\tau} $. Each temporal edge stalk $V_\ell^\tau$ is endowed with an inner product $\langle - , - \rangle_{\Lambda_\ell} := \langle - , \Lambda_\ell(-) \rangle$.
        \item For each spatial edge $v^{(t)}_{\ell+1} \xrightarrow{e^s_\ell} v^{(t)}_\ell$, $0\leq\ell\leq L-1$, an edge stalk $V^s_\ell:=V_\ell$, with source restriction map $W_\ell:V_{\ell+1}^{(t)}\to V_\ell^s$ and target restriction map $I:V_\ell^{(t)}\to V_\ell^s$. Each spatial edge stalk $V_\ell^s$ is endowed with an inner product $\langle - , - \rangle_{\Sigma_\ell} := \langle - , \Sigma_\ell(-) \rangle$.
    \end{enumerate}
\end{definition}
The time-slice sheaf $\mathcal{F}^{(t)}$ is shown in Figure~\ref{fig: time-slice sheaf}.
\begin{figure}
    \[\begin{tikzcd}
	{V_{L}^{(t-1)}} & {\color{red}{V_L^{\tau}}} & {V_L^{(t)}} \\
	&& {\color{blue}{V_{L-1}^s}} \\
	{V_{L-1}^{(t-1)}} & {\color{red}{V_{L-1}^{\tau}}} & {V_{L-1}^{(t)}} \\
	\vdots & \vdots & \vdots \\
	{V_1^{(t-1)}} & {\color{red}{V_1^{\tau}}} & {V_1^{(t)}} \\
	&& {\color{blue}{V_0^s}} \\
	{V_0^{(t-1)}} & {\color{red}{V_0^{\tau}}} & {V_0^{(t)}}
	\arrow["{A_L}", from=1-1, to=1-2]
	\arrow["I"', from=1-3, to=1-2]
	\arrow["{W_{L-1}}", from=1-3, to=2-3]
	\arrow["{A_{L-1}}", from=3-1, to=3-2]
	\arrow["I"', from=3-3, to=2-3]
	\arrow["I"', from=3-3, to=3-2]
	\arrow[from=3-3, to=4-3]
	\arrow["{A_1}", from=5-1, to=5-2]
	\arrow[from=5-3, to=4-3]
	\arrow["I"', from=5-3, to=5-2]
	\arrow["{W_0}", from=5-3, to=6-3]
	\arrow["{A_0}", from=7-1, to=7-2]
	\arrow["I"', from=7-3, to=6-3]
	\arrow["I"', from=7-3, to=7-2]
\end{tikzcd}\]
    \caption{The time-slice sheaf $\mathcal{F}^{(t)}$ associated to a linear tPC network. Temporal and spatial edge-stalks are shown in red and blue respectively. The left-most column represents the layers of the network at time $t-1$, while the black vertex stalks in the right-most column represent the layers of the network at time $t$. }
    \label{fig: time-slice sheaf}
\end{figure}

\begin{remark}
    We make two remarks about the time-slice sheaf.
    \begin{enumerate}
        \item The superscripts on the $V_j$ stalks are purely decorative, and used to further disambiguate temporal edge stalks $(-)^{\tau}$, spatial edge stalks $(-)^{s}$, and time-slice vertex stalks $(-)^{(t)}$. As a vector space, each stalk is simply the underlying $V_\ell$. 
        \item The time-slice sheaf is not a directed cellular sheaf, but a classical one. In particular, all edges are symmetric broadcaster/receivers. Since the previous equilibria $z^{\star}_{t-1}$ are all clamped during inference, there is no advantage to maintaining ``causal'' edges.
    \end{enumerate}
\end{remark}

\begin{definition}
    The \textbf{cochain spaces} associated to the time-slice sheaf are 
    \begin{align*}
        C^0(\calG^{(t)}; \calF^{(t)} )  & :=  \left ( \bigoplus_{\ell = 0}^L V_\ell^{(t-1)} \right ) \oplus \left ( \bigoplus_{\ell = 0}^L V_\ell^{(t)} \right ) \\
        C^1(\calG^{(t)}; \calF^{(t)} )  & :=  \left ( \bigoplus_{\ell = 0}^{L-1} V_\ell^{s} \right ) \oplus \left ( \bigoplus_{\ell = 0}^L V_\ell^{\tau} \right )
    \end{align*}
\end{definition}

\begin{remark}
    The cochain spaces $C^\bullet(\calG^{(t)}; \calF^{(t)} )$ each come equipped with an inner product via the direct sum. The zero cochains $C^0(\calG^{(t)}; \calF^{(t)} ) $ inherit the direct sum of all $\langle -,- \rangle$ inner products. We will always use $\langle - , - \rangle$ to refer to any direct sum of these undecorated inner products. Letting
    $$\Sigma := \mathrm{diag}(\Sigma_0, \ldots, \Sigma_{L-1}) \, , \qquad \Lambda_+ := \mathrm{diag}(\Lambda_0, \ldots, \Lambda_L) $$
    take $M_+$ to be the block-matrix
    $$M_+ := \begin{bmatrix} \Sigma & 0 \\ 0 & \Lambda_+ \end{bmatrix} \, . $$
    The inner product on $C^1(\calG^{(t)}; \calF^{(t)} )$ is exactly $\langle-,-\rangle_{M_+} := \langle-,M_+(-)\rangle$. 
\end{remark}

\begin{definition}
    The \textbf{total coboundary} $\delta_{\mathrm{tot}} : C^0(\calG^{(t)}; \calF^{(t)} ) \to C^1(\calG^{(t)}; \calF^{(t)} )$ is given by its image on each edge-stalk
    $$(\delta_{\mathrm{tot}}\xx)_\ell^\tau = \xx_\ell^{(t)} - A_\ell \xx_\ell^{(t-1)} \quad (0\leq \ell\leq L), \qquad  (\delta_{\mathrm{tot}}\xx)_\ell^s = \xx_\ell^{(t)} - W_\ell \xx_{\ell+1}^{(t)} \quad (0\leq \ell\leq L-1) \, . $$
\end{definition}

\begin{definition}
    A \textbf{global section} of $\calF^{(t)}$ is a zero cochain $\xx \in C^0(\calG^{(t)} ; \calF^{(t)})$ such that $\delta_{\mathrm{tot}} \xx = 0$. 
\end{definition}

\begin{remark}
    A global section represents a choice of activations at timeslices $t-1$ and $t$ such that there is no spatial or temporal error in the $t$ slice. 
\end{remark}

\subsection{Relative Cohomology}

\begin{definition}
    We introduce the following convenient spaces. 
    \begin{align*}
        H_f & := \bigoplus_{\ell = 1}^L V_\ell^{(t)} \\
        \mathcal{E}^s & := \bigoplus_{\ell = 0}^{L-1} V_\ell^s \\
        \mathcal{E}^\tau & := \bigoplus_{\ell = 1}^L V_\ell^\tau
    \end{align*}
\end{definition}

\begin{remark}
    The spaces $H_f$ and $\calE^\tau$ have canonically identified underlying vector spaces $\bigoplus_{\ell=1}^L V_\ell$, whereas $\calE^s$ has underlying vector space $\bigoplus_{\ell=0}^{L-1}V_\ell$. Each space carries its own direct-sum inner product. We denote the inner product on $H_f$ by an undecorated $\langle-,-\rangle$, and use $(-)^*$ for adjoints with respect to the underlying undecorated inner products. We denote the inner products on $\calE^s$ and $\calE^\tau$ by $\langle-,-\rangle_\Sigma := \langle-,\Sigma(- ) \rangle$ and $\langle-,-\rangle_\Lambda:= \langle-,\Lambda(-)\rangle$ respectively, with block-diagonal precision matrices
    $$\Sigma := \mathrm{diag}(\Sigma_0, \ldots, \Sigma_{L-1}), \qquad \Lambda := \mathrm{diag}(\Lambda_1, \ldots, \Lambda_L) \,. $$
    We write $I:H_f\to\calE^\tau$ for the canonical identification and, when no confusion can arise, use it implicitly in expressions such as $\delta v$ for $v\in\calE^\tau$.
\end{remark}

\begin{definition}
    Let $\calF^{(t)}$ be the time-slice sheaf at slice $t\geq1$ of a linear tPC network. The \textbf{boundary sub-complex} $B$ consist of the vertices $\left \{ v_{j}^{(t-1)} \right \}_{j=0}^L \cup \{v_0^{(t)} \}$, along with the single edge $e_0^\tau$.
\end{definition}

\begin{remark}
    The boundary sub-complex is the full subcomplex of $\mathcal{G}^{(t)}$ containing all data which is ``clamped'' during inference at time-slice $t$. 
\end{remark}

\begin{definition}
    The \textbf{relative cochain spaces} associated to the time-slice sheaf $\calF^{(t)}$ are
    \begin{align*}
        C^0(\calG^{(t)}, B ; \calF^{(t)} ) & := \bigoplus_{\ell = 1}^L V_{\ell}^{(t)} \\
        & = H_f\\
        C^1(\calG^{(t)}, B; \calF^{(t)} ) & := \left ( \bigoplus_{\ell = 0}^{L-1} V_\ell^s\right ) \oplus \left ( \bigoplus_{\ell = 1}^L V_{\ell}^{\tau} \right ) \\
        & = \calE^s \oplus \calE^\tau \, . 
    \end{align*}
\end{definition}

\begin{remark}
    For ease of reading, we will refer to these spaces of relative cochains through their identification with $H_f$ and $\calE^s \oplus \calE^\tau$.
\end{remark}

Recall that $\delta:H_f\to\calE^s$ denotes the spatial coboundary restricted to the free vertex stalks. Explicitly, for $z=(z_1,\ldots,z_L)^\top\in H_f$,
\begin{equation*}
    (\delta z)_0^s=-W_0z_1,
    \qquad
    (\delta z)_\ell^s=z_\ell-W_\ell z_{\ell+1}
    \quad (1\leq\ell\leq L-1).
\end{equation*}

\begin{definition}
    The \textbf{relative coboundary} operator $\calD: H_f \to \calE^s \oplus \calE^\tau$ associated to the relative cochain complexes is the block operator
    $$\mathcal{D}  := \begin{bmatrix}
        \delta \\ I
    \end{bmatrix} \, .$$
    This relative coboundary induces the \textbf{relative cochain complex}
\[\begin{tikzcd}
	{H_f} && {\calE^s \oplus \calE^\tau}
	\arrow["\calD", from=1-1, to=1-3]
\end{tikzcd}\]
    with \textbf{relative cohomology} groups
    $$ \mathcal{H}^0(\calG^{(t)}, B;\calF^{(t)}) := \ker(\calD) \, , \qquad \calH^1(\calG^{(t)}, B;\calF^{(t)}) := \frac{\calE^s \oplus \calE^\tau}{\mathrm{image}(\calD)} \, .  $$
\end{definition}

\begin{remark}
    We will refer to the relative cohomology groups $\calH^\bullet(\calG^{(t)}, B;\calF^{(t)})$ as \textbf{slice-cohomology} groups, and use the notation $\Hslice{\bullet}$.
\end{remark}

\begin{proposition}\label{pr: zero slice cohomology}
    The degree zero slice cohomology is zero; $\Hslice{0} \cong 0$.
\end{proposition}

\begin{proof}
    Since the lower block of $\calD$ is the identity map, $\calD z = 0 \iff z = 0$.
\end{proof}

\begin{proposition}
    \label{pr:H^1 computation}
    The degree one slice cohomology is isomorphic to the space of spatial edge stalks; $\Hslice{1} \cong \calE^s$.
\end{proposition}

\begin{proof}
    We construct an explicit isomorphism. Let $\chi: \calE^s \oplus \calE^\tau \to \calE^s$ denote the map
    $$\chi(u,v) = u - \delta v \, . $$
    It is straightforward to check that $\ker(\chi) = \text{image}(\mathcal{D})$ and $\text{image}(\chi) = \calE^s$. Taking the quotient of the domain by $\ker(\chi)$ yields an isomorphism
    \begin{equation*}
        \overline{\chi} : \Hslice{1} \xlongrightarrow{\cong} \calE^s \, . \qedhere
    \end{equation*}
    \renewcommand*{\jmlrQED}{\par\bigskip}
\end{proof}

\subsection{Spatiotemporal Frustration}

\begin{definition}
    Let $a_t \in \calE^\tau$ be a fixed collection of temporal edge activations, and $o_t \in V_0^{(t)}$ a fixed activation in layer zero. The \textbf{boundary excitation} associated to these data is $c_t \in \calE^s \oplus \calE^\tau$
    $$c_t := \begin{bmatrix}
        -b_t \\ -a_t
    \end{bmatrix}$$
    where $b_t \in \calE^s$ is the block vector whose $e_0^{s}$ component is $(b_t)_{e_0^s}:=-o_t$, and which is zero on all other spatial edge stalks.
\end{definition}

\begin{remark}
    We think of $o_t$ as the time $t$ input to the tPC network. Similarly, with $A:=\bigoplus_{\ell=1}^L A_\ell:H_f\to\calE^\tau$ under the canonical identification above, we think of $a_t$ as the equilibrated free activations of time slice $t-1$ pushed forward by the learned temporal map: $a_t=Az_{t-1}^\star$. However, the following analysis holds for arbitrary data $a_t$ and $o_t$ of the indicated types.
\end{remark}

\begin{definition}
    Given a boundary excitation $c_t$, the \textbf{slice residual}  $r_t(z_t) \in \calE^s \oplus \calE^\tau$ of a free activation $z_t \in H_f$ is given by
    $$r_t(z_t) := \calD z_t + c_t \, . $$
\end{definition}
Thus the two blocks of the slice residual are
\begin{align*}
    r_t^s(z_t)&=\delta z_t-b_t
    =\bigl(o_t-W_0z_{1,t},\,z_{1,t}-W_1z_{2,t},\,\ldots,\,z_{L-1,t}-W_{L-1}z_{L,t}\bigr)^\top,\\
    r_t^\tau(z_t)&=z_t-a_t.
\end{align*}

\begin{remark}
    Proposition~\ref{pr: zero slice cohomology} ($\Hslice{0} \cong 0$) implies that for any boundary excitation $c_t$, there is a unique free activation $z_t$ that minimizes the norm of the slice residual $r_t(z_t)$. On the other hand, Proposition~\ref{pr:H^1 computation} ($\Hslice{1} \cong \calE^s)$ says that the obstructions to the existence of a free activation with zero slice residual can be put in one-to-one correspondence with $\calE^s$. This latter observation motivates the following definition.
\end{remark}

\begin{definition}
    Let $c_t = \left ( \begin{bmatrix} -b_t &  -a_t \end{bmatrix} \right )^\top$ be a fixed boundary excitation. The \textbf{spatiotemporal frustration} $\kappa_t \in \calE^s$ of $c_t$ is the quantity
    $$\kappa_t := \delta a_t - b_t \, . $$
\end{definition}

\begin{remark}
    To interpret the spatiotemporal frustration, let $\chi: \calE^s \oplus \calE^\tau \to \calE^s$ be the map from the proof of Proposition~\ref{pr:H^1 computation}. Notice that for any free activation $z_t \in H_f$, applying $\chi$ to the slice residual yields
    \begin{equation}
        \chi(r_t(z_t)) = \kappa_t
    \end{equation}
    This identity holds for all activations $z_t \in H_f$. In particular, it holds for $z_t = 0$, yielding $\chi(c_t) = \kappa_t$. One recovers $\overline{\chi}([c_t]) = \kappa_t$. Hence $\kappa_t$ represents the slice cohomology class of $c_t$ in $\Hslice{1}$, and measures the obstruction to extending the boundary data with zero residual on every edge outside $B$.
\end{remark}

We formalize this remark with the following theorem.

\begin{theorem}
    The following are equivalent for a timeslice sheaf with boundary excitation $c_t$ and spatiotemporal frustration $\kappa_t$. 
    \begin{enumerate}
        \item $[c_t] = 0$; 
        \item $\kappa_t = 0$;
        \item there exists an activation $z_t \in H_f$ with zero slice residual;
        \item the boundary data inducing $c_t$ admit an extension with zero residual on every edge outside $B$;
        \item the temporal proposal $a_t$ already satisfies the current spatial constraints $\delta a_t = b_t$. 
    \end{enumerate}
\end{theorem}

\subsection{Hodge Theory and Harmonic Extension}
We now explore the relative Hodge theory of the timeslice sheaf. Recall that $\calE^s \oplus \calE^\tau$ comes equipped with an inner product via the direct sum of the inner products $\langle -,- \rangle_\Sigma$ and $\langle -,- \rangle_\Lambda$ on $\calE^s$ and $\calE^\tau$ respectively. We denote this inner product by $\langle -,- \rangle_M := \langle -,M(-) \rangle$, with block-diagonal precision matrix
$$M = \begin{bmatrix} \Sigma & 0 \\ 0 & \Lambda \end{bmatrix} \, . $$

\begin{definition}
    Let $c_t$ be a fixed boundary excitation for the timeslice sheaf $\calF^{(t)}$. The \textbf{slice energy} of $z_t \in H_f$ is the quantity
    \begin{equation}\label{eq: slice energy}
        \Eslice(z_t) := \frac{1}{2} \| \calD z_t + c_t \|_M^2
    \end{equation}
    where $\|-\|_M$ is the norm induced by $\langle -, - \rangle_M$. 
\end{definition}

\begin{definition}
    A free activation $z_t \in H_f$ is \textbf{harmonic} if and only if it minimizes the slice energy $\Eslice(z_t)$. 
\end{definition}

\begin{definition}
    The \textbf{sheaf Laplacian} $\calL_{\mathrm{tot}}: C^0(\calG^{(t)}; \calF^{(t)} )\to C^0(\calG^{(t)}; \calF^{(t)} )$ of the time-slice sheaf $\calF^{(t)}$ is the operator
    $$\calL_{\mathrm{tot}} := \delta_{\mathrm{tot}}^* M_+ \delta_{\textrm{tot}} \, . $$
\end{definition}



\begin{remark}
    Suppose we order the summands of $C^0(\calG^{(t)}; \calF^{(t)} )$ and $C^1(\calG^{(t)}; \calF^{(t)} )$ as
    \begin{align} 
        C^0(\calG^{(t)}; \calF^{(t)}) & = \left [ H_f \right ]  \oplus \left [ \left ( \bigoplus_{\ell = 0}^L V_{\ell}^{(t-1)} \right )  \oplus V_0^{(t)}\right ] \label{eq: C^0 decomp}\\
        C^1(\calG^{(t)}; \calF^{(t)} ) & = \left [ \calE^s \oplus \calE^\tau \right ] \oplus \left [ V_0^\tau\right ]  \label{eq: C^1 decomp} \, . 
    \end{align}
    Write $U$ for the vertex and edge components not contained in $B$. This decomposition separates the components indexed by $U$ from those indexed by $B$.
    With respect to these blocks, 
    The sheaf Laplacian inherits a block structure
    $$\calL_{\mathrm{tot}} = \begin{bmatrix}
        \calL_{\mathrm{tot}}^{UU} & \calL_{\mathrm{tot}}^{UB} \\
        \calL_{\mathrm{tot}}^{BU} & \calL_{\mathrm{tot}}^{BB}
    \end{bmatrix}$$
    One may check that
    \begin{equation}\label{eq: LUU}
        \calL_{\mathrm{tot}}^{UU} = \calD^* M \calD =\calL+ \Lambda 
    \end{equation}
    where $\calL = \delta^* \Sigma \delta$ is the spatial Laplacian on $H_f$.
\end{remark}

\begin{definition}
    Let
    $$s_t:=\bigl((z_{\ell,t-1}^\star)_{\ell=0}^L,o_t\bigr)^\top
    \in\left(\bigoplus_{\ell=0}^L V_\ell^{(t-1)}\right)\oplus V_0^{(t)}$$
    be a boundary condition. Set $a_t=(A_\ell z_{\ell,t-1}^\star)_{\ell=1}^L$, and let $c_t$ be the induced boundary excitation. A \textbf{harmonic extension} of $s_t$ is a collection of free activations $z_t^\star \in H_f$ such that, for some $x$,
    $$\begin{bmatrix}
        \calL_{\mathrm{tot}}^{UU} & \calL_{\mathrm{tot}}^{UB} \\
        \calL_{\mathrm{tot}}^{BU} & \calL_{\mathrm{tot}}^{BB}
    \end{bmatrix}
    \begin{bmatrix}
        z_t^\star \\ s_t
    \end{bmatrix} =
    \begin{bmatrix}
        0 \\
        x
    \end{bmatrix} \, . $$
\end{definition}

\begin{proposition}
    Every choice of $s_t$ admits a unique harmonic extension.
\end{proposition}

\begin{proof}
    By definition, $z_t^\star$ is a harmonic extension of $s_t$ if and only if
    \begin{equation}\label{eq: harmonicity}
        \calL_{\mathrm{tot}}^{UU} z_t^\star + \calL_{\mathrm{tot}}^{UB} s_t = 0 \, . 
    \end{equation} Equation~\eqref{eq: LUU} ensures that $\calL_{\mathrm{tot}}^{UU}$ is invertible, guaranteeing the existence of a unique harmonic extension. 
\end{proof}

\begin{definition}
    The \textbf{diffusive operator} $\mathfrak{G}:\calE^s \oplus \calE^\tau \to H_f$ and \textbf{harmonic projector} $\mathfrak{H}: \calE^s \oplus \calE^\tau  \to \calE^s \oplus \calE^\tau $ of the time-slice sheaf $\calF^{(t)}$ are given by 
    \begin{align}
        \mathfrak{G} &= (\calL_{\mathrm{tot}}^{UU})^{-1} \calD^* M \\
        \mathfrak{H} & = I - \calD\mathfrak{G} \, . 
    \end{align}
\end{definition}

\begin{proposition}
    Let $s_t$ be a boundary condition with corresponding boundary excitation $c_t$. The following are equivalent.
    \begin{enumerate}
        \item $z_t^\star$ is the harmonic extension of $s_t$;
        \item $z_t^\star$ minimizes the slice energy $\Eslice(z_t)$ over $z_t \in H_f$;
        \item $z_t^\star = - \mathfrak{G} c_t$. 
    \end{enumerate}
\end{proposition}

\begin{proof} \newline
    $\mathbf{1 \iff 3:}$ Following Equation~\eqref{eq: harmonicity}, the unique harmonic extension $z_t^\star$ can be written as
    $$z_t^\star = -(\calL_{\mathrm{tot}}^{UU})^{-1} \calL_{\mathrm{tot}}^{UB} s_t \, . $$
    With respect to the same decompositions of the cochain spaces (Equation~\eqref{eq: C^0 decomp} and Equation~\eqref{eq: C^1 decomp}), write the total coboundary and precision matrix $M_+$ as 
    $$\delta_{\mathrm{tot}} = \begin{bmatrix}
        \delta_{\mathrm{tot}}^{UU} & \delta_{\mathrm{tot}}^{UB} \\
        0 & \delta_{\mathrm{tot}}^{BB}
    \end{bmatrix} \, \, \qquad M_+ = \begin{bmatrix}
        M & 0 \\ 0 & \Lambda_0
    \end{bmatrix} . 
    $$
    The lower-left block $\delta_{\mathrm{tot}}^{BU}$ is zero because $B$ is a sub-complex. The Laplacian component $\calL_{\mathrm{tot}}^{UB}$ can therefore be written as
    $$\calL_{\mathrm{tot}}^{UB} = (\delta_{\mathrm{tot}}^{UU})^*M\delta_{\mathrm{tot}}^{UB} \, . $$
    Moreover, one may straightforwardly check that $\delta_{\mathrm{tot}}^{UU} = \mathcal{D}$ is the relative coboundary, and that $\delta_{\mathrm{tot}}^{UB}s_t = c_t$. It follows that
    \begin{align*}
        z_t^\star & = -(\calL_{\mathrm{tot}}^{UU})^{-1} \calL_{\mathrm{tot}}^{UB} s_t \\
        & = -(\calL_{\mathrm{tot}}^{UU})^{-1} \calD^* M c_t \\
        & = - \mathfrak{G} c_t \, . 
    \end{align*}
    $\mathbf{2 \iff 3:}$ The slice energy (Equation~\eqref{eq: slice energy}) of an activation $z_t$ is given by
    $$\frac{1}{2}\| \mathcal{D}z_t + c_t \|_M^2 \, . $$
    Minimizing the slice energy over $z_t$ is therefore equivalent to finding the OLS solution to
    $$\mathcal{D}z_t = - c_t$$
    with respect to the metric $M$. Since $\mathcal{D}$ is injective, the OLS solution is given by
    \begin{align*}
        z_t^\star & = -(\calD^* M \calD)^{-1} \calD^*M  c_t \\
        & = - \mathfrak{G}c_t \, . \qedhere
    \end{align*}
    \renewcommand*{\jmlrQED}{\par\bigskip}
\end{proof}

\begin{remark}
    This proposition clarifies the relationship between harmonic extension in the timeslice sheaf and tPC inference. As the unique global minimizer of the convex slice energy $\Eslice(z_t)$, the harmonic extension $z_t^\star$ is the unique critical point of the dynamics
    \begin{align*}
        \dot{z}_t & = - \nabla_{z_t} \Eslice(z_t) \\
        & = - \delta^* \Sigma(\delta z_t - b_t) - \Lambda(z_t - a_t) \, . 
    \end{align*}
    Therefore the harmonic extension $z_t^\star$ is exactly the tPC inference equilibrium. 
\end{remark}

\begin{definition}
    Let $s_t$ be a boundary condition with harmonic extension $z_t^\star$. The \textbf{extension residual} $r_t^\star$ is the quantity
    $$r_t^\star := \calD z_t^\star + c_t \, . $$
\end{definition}

\begin{remark}
    The extension residual measures the failure of $z_t^\star$ to extend $s_t$ with zero residual on every edge outside $B$. In particular, such a zero-residual extension exists if and only if $r_t^\star = 0$.
\end{remark}

\begin{corollary}
        Let $s_t$ be a boundary condition with corresponding boundary excitation $c_t$. The following are equivalent.
    \begin{enumerate}
        \item $r_t^\star$ is the extension residual of $s_t$;
        \item $r_t^\star = \mathfrak{H} c_t$. 
    \end{enumerate}
\end{corollary}

\begin{remark}
    The extension residual can be written in terms of the spatiotemporal frustration $\kappa_t$. Letting $\mathcal{K} := ( \Sigma^{-1} + \delta \Lambda^{-1} \delta^*)^{-1}$, one may write
    $$r_t^\star = \begin{bmatrix} \Sigma^{-1} \\ - \Lambda^{-1} \delta^*\end{bmatrix} \mathcal{K} \kappa_t \, . $$
    This observation strengthens the interpretation of $\kappa_t$ as a measurement of irreducible error arising from spatiotemporal tension.
\end{remark}

\subsection{Harmonic-Diffusive Overlap and Learning}
At the learning step, the size of the gradient is determined by the harmonic and diffusive operators. Holding equilibrated free activations $z_{t-1}^\star$ and $z_{t}^\star$ fixed, the following gradients of the slice energy are recovered for the structure maps: 
\begin{align*}
    \nabla_{W_\ell} E_{\text{slice}} &= \Sigma_\ell\left (W_\ell z_{\ell+1,t}^\star - z_{\ell,t}^\star \right ) (z_{\ell+1,t}^\star)^* \\
    & = - \Sigma_\ell(r_t^\star)_{e^s_{\ell}} (z_{\ell+1,t}^\star)^* \qquad (0\leq\ell\leq L-1), \\
    \nabla_{A_\ell} E_{\text{slice}} &= \Lambda_\ell\left ( A_\ell z_{\ell,t-1}^\star - z_{\ell,t}^\star \right )  (z_{\ell,t-1}^\star)^* \\
    & = - \Lambda_\ell (r_t^\star)_{e^\tau_{\ell}} (z_{\ell,t-1}^\star)^* \qquad (1\leq\ell\leq L) \, . 
\end{align*}
In these equations, we adopt the convention that $z^\star_{0,t} = o_t$. 

Hence learning requires both a non-vanishing residual and a non-zero activation. Writing $r_t^\star = \mathfrak{H}_t c_t$ and $z_t^\star = - \mathfrak{G}_t c_t$, where the subscript on the harmonic projector and diffusive operator indicates which time-slice they act on, one obtains
\begin{align*}
    \nabla_{W_\ell} E_{\text{slice}} & =  \Sigma_\ell \left ( \mathfrak{H}_t c_t \right )_{e_\ell^s} \left ( \mathfrak{G}_t c_t \right )_{\ell + 1}^* \\ 
    \nabla_{A_\ell} E_{\text{slice}} & =  \Lambda_\ell \left ( \mathfrak{H}_t c_t \right )_{e_\ell^\tau} \left ( \mathfrak{G}_{t-1} c_{t-1} \right )_{\ell}^* \, .
\end{align*}
Hence learning occurs precisely on edges where both $\mathfrak{G}_\bullet c_\bullet$ and $\mathfrak{H}_\bullet c_\bullet$ are non-negligible. 

\subsection{Recovering the Persistence and Exposure Operators}
We now show where the persistence and exposure operators $P$ and $Q$ emerge naturally in the algebraic account of tPC inference. We begin by decomposing the diffusive operator $\mathfrak{G}: \calE^s \oplus \calE^\tau \to H_f$ as a block operator over the domain:
\begin{align*}
    \mathfrak{G} &= (\calL_{\mathrm{tot}}^{UU})^{-1} \calD^*M \\
    &= (\calL + \Lambda)^{-1} \begin{bmatrix} \delta^* & I \end{bmatrix} \begin{bmatrix}
        \Sigma & 0 \\
        0 & \Lambda 
    \end{bmatrix} \\
    & = (\calL + \Lambda)^{-1} \begin{bmatrix} \delta^* \Sigma & \Lambda \end{bmatrix}  \\
    & = \begin{bmatrix} (\calL + \Lambda)^{-1}  \delta^* \Sigma & P\end{bmatrix} \, . 
\end{align*}

\begin{remark}
    $P := (\calL + \Lambda)^{-1} \Lambda$ is exactly the $\calE^\tau \to H_f$ block of the diffusive operator $\mathfrak{G}$, which in turn is the Moore-Penrose pseudoinverse of $\calD$. 
\end{remark}

This decomposition of $\mathfrak{G}$ gives a corresponding decomposition of the harmonic operator $\mathfrak{H}$:
\begin{align*}
    \mathfrak{H} &:= I - \mathcal{D}\mathfrak{G} \\
    & = \begin{bmatrix}
        I & 0 \\
        0 & I
    \end{bmatrix} - 
    \begin{bmatrix}
        \delta \\ I
    \end{bmatrix}
    \begin{bmatrix}
        (\calL + \Lambda)^{-1}  \delta^* \Sigma & P
    \end{bmatrix} \\
    & = \begin{bmatrix}
        I - \delta  (\calL + \Lambda)^{-1}  \delta^* \Sigma & - \delta P \\
         -(\calL + \Lambda)^{-1}  \delta^* \Sigma & I - P
    \end{bmatrix} \, . 
\end{align*}
From the identity $I = P + Q$, one recovers $Q$ in the lower right block of $\mathfrak{H}$.

\begin{remark}
    $Q := (\calL + \Lambda)^{-1} \calL$ is exactly the $\calE^\tau \to \calE^\tau$ block of the harmonic projector $\mathfrak{H}$, which in turn is the orthogonal projection of $\calE^s \oplus \calE^\tau$ onto the orthogonal complement $\mathrm{image}(\calD)^\perp$ under the inner product $\langle - , - \rangle_M$. 
\end{remark}

\section{Delayed-Recall Experiment}
\label{app:paper-native-tpc}

\paragraph{Inputs and Inference}
For each seed, we randomly sample and fix an orthonormal eigenbasis of the
spatial sheaf Laplacian $\mathcal L=W^\top W$ satisfying
\begin{equation*}
  \mathcal L v_i=\ell_i v_i,
  \qquad
  \langle v_i,v_j\rangle=\delta_{ij}.
\end{equation*}
Define $v_C:=v_1$ as the cue direction, $v_2,\ldots,v_5$ as the four
distractor directions, and $v_E:=v_N=v_{16}$ as the exposed direction.  The
remaining carrier directions receive no distractor input.

We draw the cue $c\sim\mathcal N(0,1)$ and scale its observation by
$\alpha=(1+\varepsilon)/\sqrt{\varepsilon}=10.1$, so that inference gives
$h_0^*=cv_C$.  For $t=1,\ldots,4$, we scale independent Gaussian distractor
observations by $\alpha$ so that their post-inference contributions along
$v_2,\ldots,v_5$ are distributed as $\mathcal N(0,\sigma_d^2I_4)$.  We choose
$\sigma_d=0.1$ to introduce nontrivial interference without overwhelming the
unit-variance cue.  Finally, the training response has amplitude $\alpha c$
along the observation direction corresponding to $v_E$, with zero input
elsewhere.  Starting from
$h_{-1}^*=0$, exact inference at each slice gives
\begin{equation*}
  h_t^*=(I+\mathcal L)^{-1}
  \left(W^\top x_t+Ah_{t-1}^*\right).
\end{equation*}
Here $x_t$ denotes the clamped observation at slice $t$.

\paragraph{Learning Comparison}
Both full temporal credit and standard tPC use the same inference procedure.
Their learning gradients differ.  Writing
$\varepsilon_{b,t}=Ah_{b,t-1}^*-h_{b,t}^*$, standard tPC treats the inferred
states as fixed and uses the gradient
\begin{equation*}
  G_{\mathrm{std}}
  =\frac1B\sum_{b=1}^B\sum_{t=1}^R
    \varepsilon_{b,t}(h_{b,t-1}^*)^\top,
\end{equation*}
to update $A$.  By contrast, full temporal credit differentiates the summed
slice energy through the entire inferred trajectory by backpropagation through
time (BPTT).  Thus, both rules see the full slice energy, but only full credit
propagates its effect to earlier transitions.  We use minibatches of $64$,
$10^4$ Adam steps, and the constraint $\lVert A\rVert_2\leq1$.  Learning rates
are tuned independently by grid search, giving $10^{-2}$ for full credit and
$3\times10^{-4}$ for standard tPC.  We use $15$ paired seeds.
For each seed, we initialise $A=0.98I+0.01U$ with orthogonal $U$, clip its
initial singular values at $0.99$, and enforce $\lVert A\rVert_2\leq1$ by
clipping singular values at one after each Adam update.

\paragraph{Evaluation}
At test time, we set the response observation to zero and use the fixed readout
\begin{equation*}
  \widehat y_R=v_E^\top(Ah_{R-1}^*-h_R^*).
\end{equation*}
We measure test MSE over $2048$ independently sampled sequences.  To examine
the effect of the cue, we compare paired runs with identical distractors and
either a sampled cue or a zero cue.  Define
\begin{equation*}
  \delta_t=h_t^{(c)}-h_t^{(0)},
  \qquad
  Z=\mathbb E\lVert\delta_0\rVert_2^2.
\end{equation*}
The plotted transmission energies are
\begin{equation*}
  z_{C\to C}(t)=\frac{\mathbb E\lVert\Pi_CA\Pi_C\delta_{t-1}\rVert_2^2}{Z},
  \qquad
  z_{C\to E}(t)=\frac{\mathbb E\lVert\Pi_EA\Pi_C\delta_{t-1}\rVert_2^2}{Z}.
\end{equation*}
Here $\Pi_C$ and $\Pi_E$ project onto the carrier and exposure subspaces, and
$Z$ normalises by the initial cue-dependent energy.  Thus, the two quantities
measure recurrent transmission of the cue from $C$ into $C$ and $E$,
respectively.  We also report the delay-averaged carrier energy
$\overline e_C=(4Z)^{-1}\sum_{t=1}^4
\mathbb E\lVert\Pi_C\delta_t\rVert_2^2$.

\paragraph{Results}
After $10^4$ steps, we obtain
\begin{center}
\begin{tabular}{lcc}
\hline
& Full temporal credit & Standard tPC \\
\hline
Median test MSE & $0.00548$ & $0.97057$ \\
Median $\overline e_C$ & $0.94012$ & $0.00121$ \\
Median $z_{C\to E}(R)$ & $0.89095$ & $3.74\times10^{-6}$ \\
Median $\max_{t<R}z_{C\to E}(t)$ & $1.48\times10^{-5}$ & $4.29\times10^{-5}$ \\
Largest $C\to E$ event occurs at response & $15/15$ seeds & $0/15$ seeds \\
\hline
\end{tabular}
\end{center}
This empirical result shows that mode routing occurs in a full-credit setting
but not in standard local tPC.

\begin{figure}[H]
  \centering
  \input{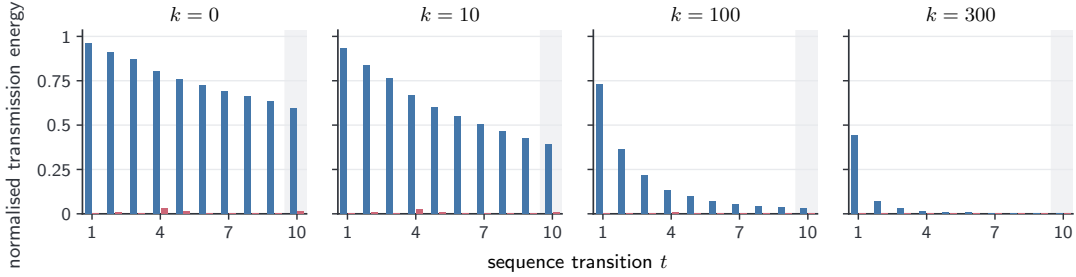}
  \caption{As in delayed recall, full temporal credit retains the digit in the
  carrier and exposes it at the response, whereas standard local tPC does not.}
  \label{fig:dense-copy-routing}
\end{figure}

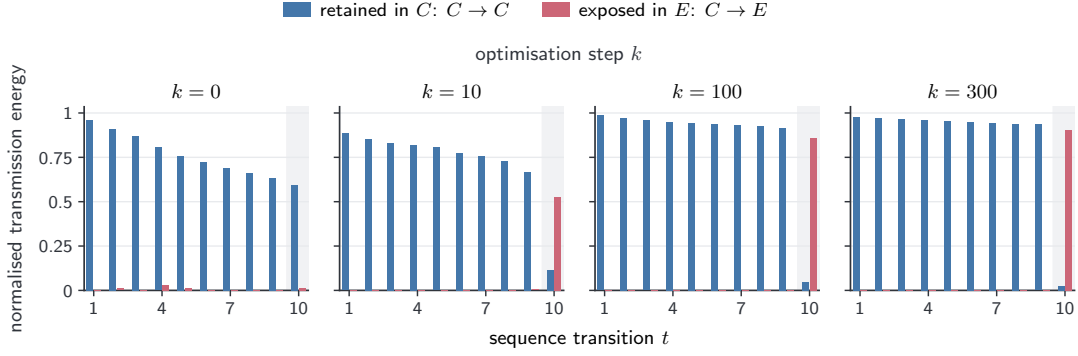
\raggedbottom
\section{Overlapping 30-digit Copy}
\label{app:copy-task}
\paragraph{Task.}
We encode 30 digits $d_j$ as scalars $c_j=(d_j-4.5)/\sqrt{8.25}$.
Digit $c_j$ arrives at $t=j$ and is due at $t=j+10$, so input and answer
periods overlap.

\paragraph{Inputs and inference.}
As in the delayed mode-routing experiment, we use one observed and one hidden
layer. Here both have 41 coordinates: $e_{40}$ is the input direction,
$C=\operatorname{span}\{e_1,\ldots,e_{40}\}$ is the carrier, and $v_E=e_0$
is the exposed response direction. We fix
$W_{\mathrm{in}}=\operatorname{diag}(\varepsilon^{-1/2},
\varepsilon^{1/2},\ldots,\varepsilon^{1/2})$ with $\varepsilon=0.001$ and
$\alpha=(1+\varepsilon)/\sqrt{\varepsilon}$. Thus
$p_C=q_E=(1+\varepsilon)^{-1}$, while an observation $\alpha c_j$ contributes
$c_j$ to its hidden coordinate when the temporal proposal is zero. During
training we clamp
\begin{equation*}
  o_t=\begin{cases}
    \alpha c_t e_{40}, & 0\leq t<10,\\
    \alpha(c_t e_{40}+c_{t-10}e_0), & 10\leq t<30,\\
    \alpha c_{t-10}e_0, & 30\leq t<40.
  \end{cases}
\end{equation*}
Starting from $h_{-1}^*=0$, we infer each $h_t^*$ using the same slice
energy and equilibrium as in the delayed mode-routing experiment.

\paragraph{Learning comparison.}
Only $A$ is learned, using the same local-tPC and full-credit BPTT rules as
in the delayed mode-routing experiment. For this task we use minibatches of
16, 300 Adam updates, $\lVert A\rVert_2\leq1$, and learning rates $0.01$
(full credit) and $0.0003$ (local tPC), selected from 100-update pilots. We
report five paired seeds with matched initial $A$, training streams, and test
digit lists.

\paragraph{Evaluation and routing.}
At test time, with $A$ fixed and $h_{-1}^*=0$, we set
$o_t=\alpha c_t e_{40}$ for $t<30$ and $o_t=0$ thereafter. At $t=10+j$
we read $\widehat c_j=v_E^\top(Ah_{t-1}^*-h_t^*)$ and compute MSE against
the 30 original digits over 1,024 new lists per seed. For the routing
diagnostic, we compare runs with $o_0=\alpha e_{40}$ and $o_0=0$, setting all
later observations to zero. Their hidden-state difference satisfies
$\delta_0=e_{40}$ and $\delta_t=PA\delta_{t-1}$. We plot the same
$z_{C\to C}$ and $z_{C\to E}$ transmission energies as in the delayed
mode-routing experiment, now with $Z=1$ and $t=1,\ldots,10$.
The delay-averaged carrier energy is
$\overline e_C=\frac{1}{9}\sum_{t=1}^{9}\lVert\Pi_C\delta_t\rVert_2^2$.

\paragraph{Results}
After 300 updates, we obtain
\begin{center}
\begin{tabular}{lcc}
\hline
& Full temporal credit & Standard tPC \\
\hline
Median test MSE & $0.01141$ & $0.99939$ \\
Median $\overline e_C$ & $0.95611$ & $0.05997$ \\
Median $z_{C\to E}(10)$ & $0.90630$ & $8.23\times10^{-5}$ \\
Median $\max_{t<10}z_{C\to E}(t)$ & $4.16\times10^{-4}$ & $0.00584$ \\
Largest $C\to E$ event occurs at response & $5/5$ seeds & $0/5$ seeds \\
\hline
\end{tabular}
\end{center}

\phantomsection\label{jmlrend}
\end{document}